\documentclass[sigconf]{acmart}
\setcopyright{none}
\usepackage{amsthm}
\usepackage{times}
\usepackage{helvet}
\usepackage{booktabs}
\usepackage{subcaption}

\usepackage{amssymb}
\usepackage{amsmath,amsfonts}
\usepackage{courier}
\usepackage{algorithm}
\usepackage{algpseudocode}
\usepackage{wrapfig}
\usepackage{multicol}
\usepackage{enumitem}
\usepackage{graphicx} 
\usepackage{wrapfig}
\usepackage[toc,title,page]{appendix}

\usepackage{xcolor}
\usepackage{xspace}

\newcommand{\mat}[1]{\mathbf{#1}}
\newcommand{\fedavg}{FedAVG\xspace}
\newtheorem{theorem}{Theorem}[section]
\newtheorem{proposition}[theorem]{Proposition}
\usepackage{tikz}
\usepackage{tikz-network}
\usetikzlibrary{decorations.pathreplacing,calc,shapes,positioning,arrows,arrows}
\usepackage{pgfplots}
\pgfplotsset{%
	,compat=1.12
	,every axis x label/.style={at={(current axis.right of origin)},anchor=north west}
	,every axis y label/.style={at={(current axis.above origin)},anchor=north east}
}
\usetikzlibrary{datavisualization,shapes.geometric,arrows.meta,decorations.markings, fit}
\usetikzlibrary{datavisualization.formats.functions}
\definecolor{scarlet}{rgb}{1.0, 0.13, 0.0}
\definecolor{brightmaroon}{rgb}{0.76, 0.13, 0.28}
\definecolor{mediumturquoise}{rgb}{0.28, 0.82, 0.8}
\definecolor{fandango}{rgb}{0.71, 0.2, 0.54}
\definecolor{antiquewhite}{rgb}{0.98, 0.92, 0.84}
\definecolor{babyblue}{rgb}{0.54, 0.81, 0.94}
\definecolor{brilliantlavender}{rgb}{0.96, 0.73, 1.0}
\definecolor{bronze}{rgb}{0.8, 0.5, 0.2}
\definecolor{cornsilk}{rgb}{1.0, 0.97, 0.86}
\definecolor{lavenderpink}{rgb}{0.98, 0.68, 0.82}
\definecolor{sandybrown}{rgb}{0.96, 0.64, 0.38}
\definecolor{celadon}{rgb}{0.67, 0.88, 0.69}

\newcommand{\highlight}[1]{\underline{\textbf{#1}}}

\DeclareMathOperator*{\concat}{%
    \mathchoice%
        {\Big\Vert}%
        {\big\Vert}%
        {\Vert}%
        {\Vert}%
}
\author{Qiying Pan}
\email{sim10\_arity@sjtu.edu.cn}
\affiliation{%
  \institution{Shanghai Jiao Tong University}
  \country{}
}

\author{Ruofan Wu}
\email{ruofan.wrf@antgroup.com}
\affiliation{%
  \institution{Ant Group}
  \country{}
}

\author{Tengfei Liu}
\email{aaron.ltf@antgroup.com}
\affiliation{%
  \institution{Ant Group}
  \country{}
}

\author{Tianyi Zhang}
\email{zty113091@antgroup.com}
\affiliation{%
  \institution{Ant Group}
  \country{}
}

\author{Yifei Zhu}
\authornote{Corresponding author}
\email{yifei.zhu@sjtu.edu.cn}
\affiliation{%
  \institution{Shanghai Jiao Tong University}
  \country{}
}

\author{Weiqiang Wang}
\email{weiqiang.wwq@antgroup.com}
\affiliation{%
  \institution{Ant Group}
  \country{}
}

\begin{document}

\title{FedGKD: Unleashing the Power of Collaboration in Federated Graph Neural Networks}

\begin{abstract}

Federated training of Graph Neural Networks (GNN) has become popular in recent years due to its ability to perform graph-related tasks under data isolation scenarios while preserving data privacy. However, graph heterogeneity issues in federated GNN systems continue to pose challenges. Existing frameworks address the problem by representing local tasks using different statistics and relating them through a simple aggregation mechanism. However, these approaches suffer from limited efficiency from two aspects: low quality of task-relatedness quantification and inefficacy of exploiting the collaboration structure. To address these issues, we propose FedGKD, a novel federated GNN framework that utilizes a novel client-side graph dataset distillation method to extract task features that better describe task-relatedness, and introduces a novel server-side aggregation mechanism that is aware of the global collaboration structure. We conduct extensive experiments on six real-world datasets of different scales, demonstrating our framework's outperformance.

\end{abstract}
\maketitle

\section{Introduction}
Federated training of Graph Neural Networks (GNN) has gained considerable attention in recent years due to its ability to apply a widely used privacy-preserving framework called Federated Learning (FL) \cite{mcmahan2017communication, kairouz2021advances} to GNN training. This approach facilitates collaborative training among isolated datasets while preserving data privacy. The protocol generally follows a typical federated learning approach, whereby the server broadcasts a global GNN model to each client. Upon receiving the model, each client trains the model using its local graph and transmits its local model to the server, which then aggregates the local weights to obtain a new global model and initiates another round of the broadcast-training-aggregation procedure. Throughout the process, the client-side graphs are kept locally, thus safeguarding privacy. Federated GNNs have successfully addressed the challenge of isolating graph data in real-life scenarios where organizations and corporations collect their private graphs, restricting others from accessing them while seeking a collaborative training approach to improve their personalized model performance.

However, the effectiveness of federated GNN frameworks is limited by graph heterogeneity issues. The isolated graph datasets exhibit significant variation, and thus, simple aggregation of all local models leads to a degradation of model performance \cite{ramezani2021learn}. The recent developments of personalized federated learning (PFL) \cite{smith2017federated,arivazhagan2019federated,xu2023personalized,bao2023optimizing} have explored mechanisms that adapt to client heterogeneity by allowing each client to train their personalized model while encouraging inter-client collaboration through properly handling the collaboration structure among clients.

Although the PFL framework is effective in mitigating statistical heterogeneity issues, it is noteworthy that graph heterogeneity issues are more severe than that in the independent setting. This is primarily due to the complex nature of graphs that incorporate topology information. In particular, graph heterogeneity encompasses not only the distributional discrepancy of node features but also the presence of disparate graph structures \cite{ramezani2021learn}. To address this challenge, a range of graph-related solutions have been proposed. 
However, the existing solutions for graph heterogeneity issues suffer from limited efficiency due to two reasons:

{
\textbf{Quality of task-relatedness quantification}: It has been shown in recent studies \cite{bao2023optimizing,pmlr-v195-zhao23b} that high-quality characterizations of the collaboration structure among clients are crucial for the generalization performance in PFL. Existing works on PFL over graphs rely on weights and gradients \cite{xie2021federated} or graph embeddings with random input graphs \cite{baek2022personalized}. The former involves similarity computation of high-dimensional random vectors, which may be imprecise due to the curse of dimensionality \cite{hastie2009elements}, while the latter one discards information related to local labels and losses expressivity.

\textbf{Efficacy of exploiting collaboration structure}: Another reason for unsatisfactory efficiency is the inadequate handling of the collaboration structure, which is typically reflected in the personalized aggregation mechanisms. Previous works base their aggregation procedures on \emph{pairwise relationships} among client tasks \cite{baek2022personalized}, which only exploit inter-client task relatedness \emph{locally}. It is therefore of interest to explore mechanisms that operate from a \emph{global} point of view in a principled way.
}

To address the problems, we propose a novel \highlight{Fed}erated \highlight{G}raph Neural Network with \highlight{K}ernelized aggregation of \highlight{D}istilled information (\textbf{FedGKD}) that achieves better utilization of the collaboration among clients. This framework comprises two modules: 
A \textbf{task feature extractor} that utilizes a novel graph data distillation method and a \textbf{task relator} that aggregates local models through a mechanism that is aware of the global property of the collaboration structure inferred from the outputs of the task feature extractor. 
Specifically, the task feature extractor devises a novel \emph{dynamic graph dataset distillation mechanism} that represents each local task by distilling local graph datasets into size-controlled synthetic graphs at every training round, enabling efficient similarity computation between clients while sufficiently incorporating local task information. The task relator first constructs a collaboration network from the distilled task features and relates tasks through a novel aggregation mechanism based on the network's \emph{global connectivity}, which measures the task-relatedness in a global sense.
As a summary, our contributions are: 
\begin{itemize}[leftmargin=*]
    \item We propose a task feature extractor based on a novel dynamic graph data distillation method, representing each local task with a distilled synthetic graph generated from all the local model weights trained at each round. The task features extracted from distilled graphs contain both data and model information, while also allowing for efficient evaluation of task-relatedness.
    \item We propose a task relator that constructs a collaboration network from the distilled graphs and relates tasks by operating a novel kernelized attentive aggregation mechanism upon local weights that encodes the global connectivity of the collaboration network.
    \item We conduct extensive experiments to validate that our framework consistently outperforms state-of-the-art personalized GNN frameworks on six real-world datasets of varying scales under both overlapping and non-overlapping settings.
\end{itemize}

The paper proceeds as follows. In Section \ref{sec:rw}, we present a review of related works. This is followed by the introduction of the preliminaries on GNN and FL in Section \ref{sec:pr}. Next, in Section \ref{sec:pf}, we introduce and formulate the problem of federated GNN. The detailed design of the framework to solve this problem is presented in Section \ref{sec:design}. In Section \ref{sec:ex}, we present the experimental results. Finally, Section \ref{sec:con} concludes the paper.

\section{Related Work\label{sec:rw}}

{
\subsection{Personalized Federated Learning}
The learning procedures of homogeneous federated learning \cite{mcmahan2017communication, kairouz2021advances} are often considered as special forms of distributed optimization algorithms like local SGD \cite{woodworth2020local}. However, these methods have been shown to suffer from client heterogeneity in terms of both convergence \cite{karimireddy2020scaffold} and client-side generalization \cite{chen2021theorem,pmlr-v195-zhao23b}. Personalized federated learning approaches have primarily focused on addressing the latter issue by incorporating adaptation strategies that can be deployed at the client side, server side, or both.
\textbf{Client-side adaptation} methods typically utilize parameter decoupling paradigms that enable flexible aggregation of partial parameters \cite{arivazhagan2019federated, pillutla2022federated} or control the optimization of local objectives by regularizing towards the global optimum \cite{li2021ditto}. However, these methods often overlook the overall collaboration structure among the clients \cite{bao2023optimizing, pmlr-v195-zhao23b}. It has been shown that with a correctly-informed collaboration structure that precisely describes the task-relatedness between clients, simple procedures can achieve minimax optimal performance \cite{pmlr-v195-zhao23b}.
On the other hand, \textbf{server-side adaptation} methods aim to measure the task-relatedness among clients and derive refined aggregation mechanisms. \cite{bao2023optimizing} utilize tools from transfer learning theory to conduct an estimating procedure that clusters clients into subgroups. \cite{ye2023personalized, xu2023personalized} propose to optimize collaboration among clients on-the-fly by solving a quadratic program at the server-side during each aggregation step. It is important to note that server-side adaptation methods often involve the transmission of additional information other than the model parameters.

\subsection{Federated Graph Representation Learning}
In \cite{ramezani2021learn}, the authors showed that naively applying FedAvg to distributed GNN training will result in irreducible error under distinct client-side graph structures, which hampers convergence. A recent line of work has been attempting to adopt personalization strategies for federated learning of graph neural networks. For instance, \cite{tan2023federated} uses client-side adaptation by sharing only a sub-structure of the client-side GNN. \cite{zhang2021subgraph} equips each client with an auxiliary neighborhood generation task. \cite{xie2021federated} applies a server-side adaptation strategy that dynamically clusters clients using intermediate gradient updates. Moreover, \cite{baek2022personalized} combines client-side and server-side adaptation methods and measures task similarity using GNN outputs based on a common input random graph.
\subsection{Dataset Distillation}
The method of dataset distillation \cite{wang2018dataset} was originally proposed as a way to improve training efficiency by distilling a large dataset into a significantly smaller one while keeping model performance almost intact. Later developments generalized the approach to graph-structured data \cite{jin2022graph, jin2022condensing}. A notable property of dataset distillation is that the distilled datasets are observed to exhibit good privacy protection against empirically constructed adversaries \cite{dong2022privacy}. This empirical property has also lead to innovations in one-shot federated learning \cite{zhou2020distilled}, which is very different from the setups in PFL and is considered an orthogonal application.
}

\begin{figure*}[t]
\centering
\subfloat[Task feature extractor\label{fig:tfe}]{\includegraphics[width=0.48\textwidth]{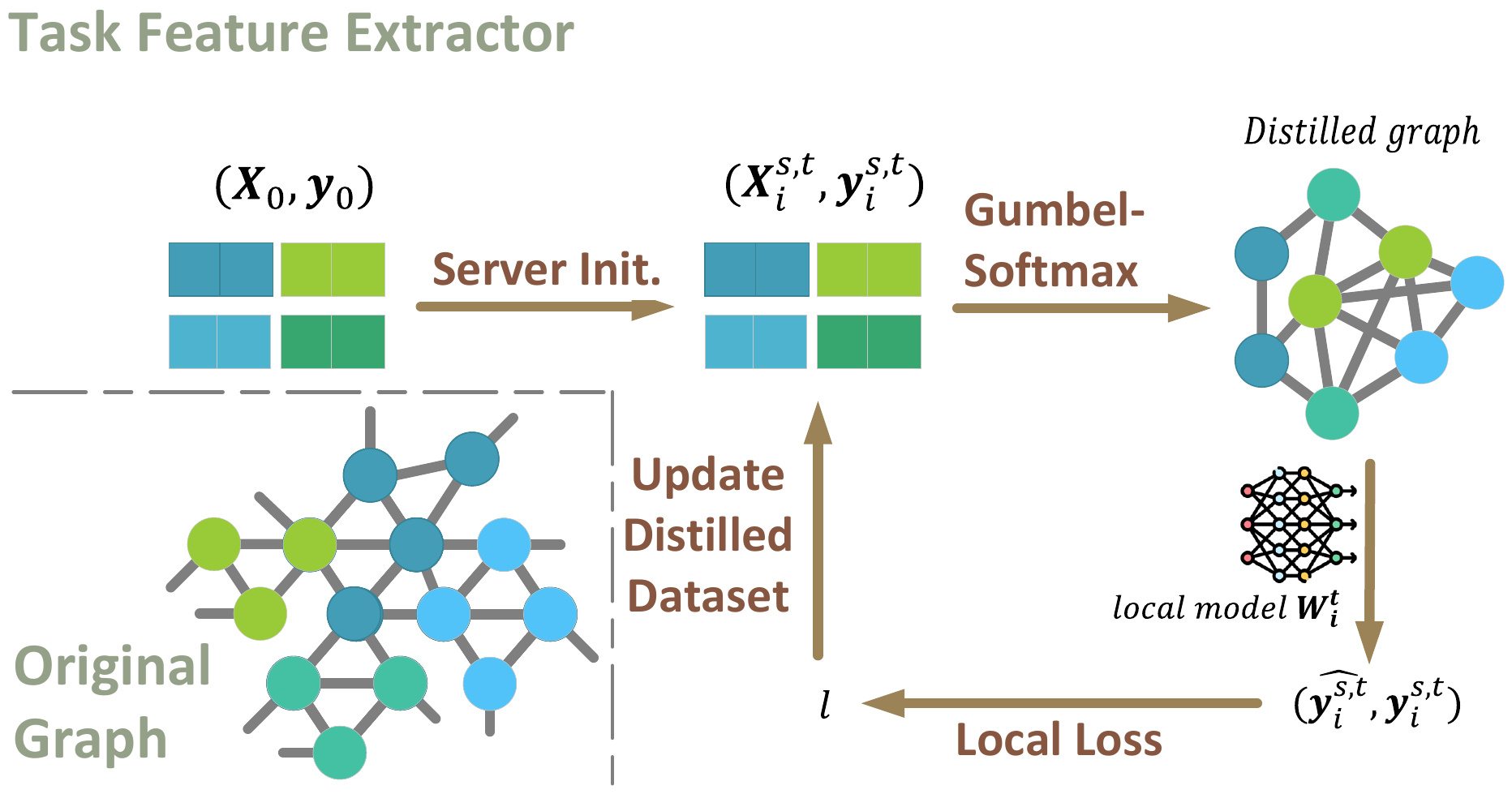}
}
\hfill
\hspace{-0.5cm}
\subfloat[Task relator\label{fig:tr}]
{\includegraphics[width=0.48\textwidth]{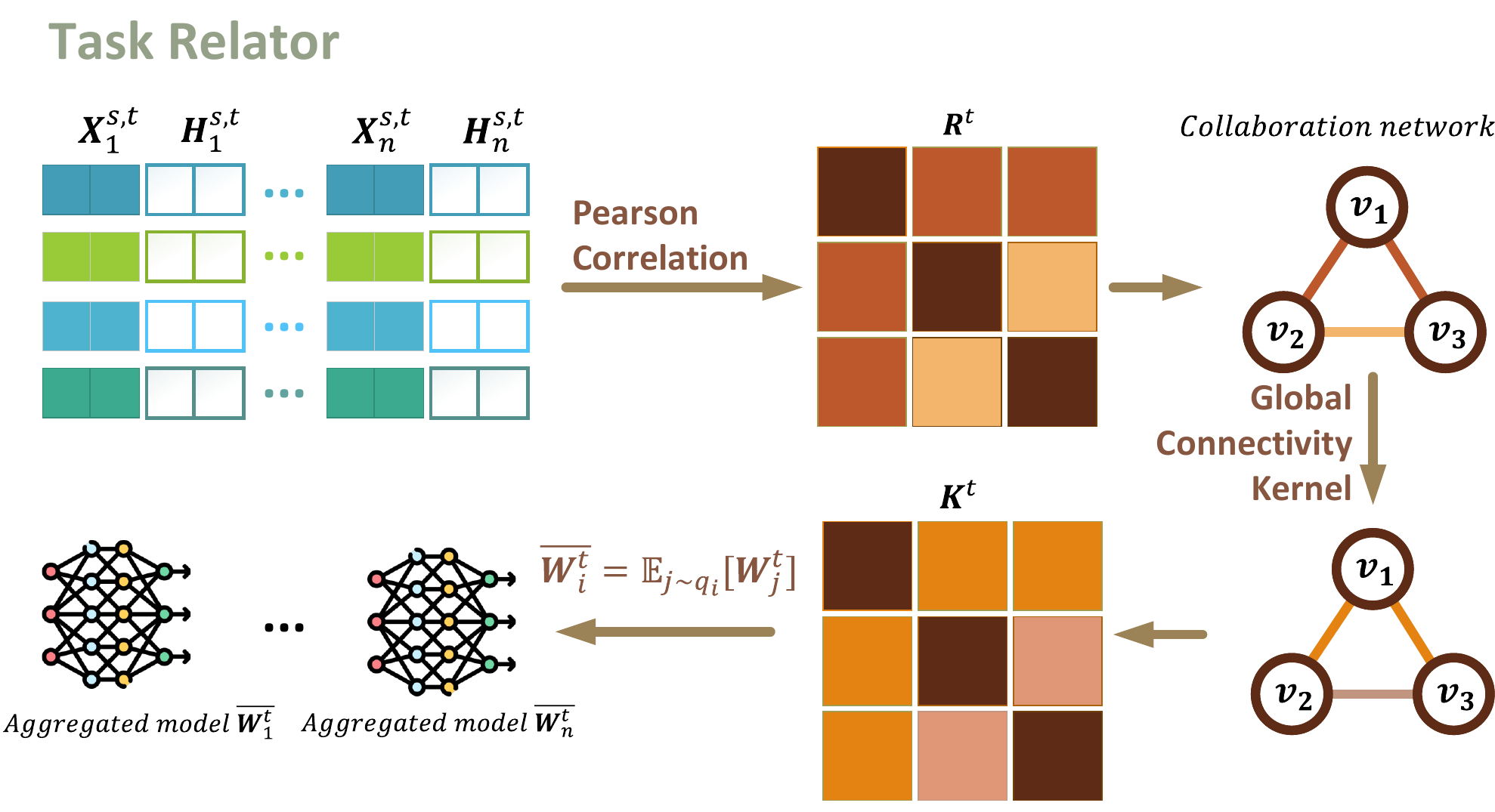}
}
\caption{Overview of two modules in the proposed FedGKD framework.}
\label{fig:overview}
\end{figure*}

\section{Preliminaries\label{sec:pr}}
In this section, we provide a brief introduction to two key concepts in our paper: Graph Neural Network and Federated Learning, which are presented in separate subsections.
\subsection{Graph Neural Network}
Consider a graph $G=(V,E)$, where $V$ represents the node set and $E$ represents the edge set. The graph is associated with a node feature matrix $\mathbf{X}\in \mathbb{R}^{|V|\times D}$. We can use Graph Neural Networks (GNN) to embed nodes in the graph with low dimensional vectors. 
An $L$-layer GNN in the message passing form \cite{pmlr-v70-gilmer17a,xu2018powerful} is recursively defined in \eqref{eq:gnn},
where $\mathbf{h}_u^l$ represents the embedding of node $u$ output from the $l-$th GNN layer, $UPD$ is a function that generates embeddings based on the former layer outputs, $AGGR$ is a function that aggregates the embeddings together and $\mathcal{N}(u)$ represents the set of neighboring nodes of node $u$.
\begin{equation}
\label{eq:gnn}
\begin{aligned}
        \mathbf{h}_u^l = UPD^l [\mathbf{h}_u^{l-1},AGG^l(\{\mathbf{h}_v^{l-1}:v\in \mathcal{N}(u) \})].
\end{aligned}
\end{equation}
In general, the outputs of $L-$th GNN layers are passed through a customized READOUT layer to accomplish various graph-related tasks. For instance, in node classification tasks, 
a simple READOUT layer can be selected as a linear layer with the number of categories as the output dimension.

\subsection{Federated Learning}
Federated Learning (FL) was introduced to address data isolation issues while preserving privacy \cite{mcmahan2017communication}. It enables collaborative training among clients without exposing raw datasets. A typical FL framework consists of three stages: (1) \textit{model initialization}, where the server broadcasts initial model weights to all clients; (2) \textit{local training}, where each client trains a local model using the initial model weights and its own dataset, and uploads the local model to the server; (3) \textit{global aggregation}, wherein the server aggregates the local models into one or more new models and broadcasts the result to each client to initiate the next training round. The FL procedure typically alternates between stage (2) and (3) until convergence.

\section{Problem Formulation\label{sec:pf}}
A personalized federated GNN framework aims to collaboratively learn local GNN models in a privacy-preserving manner. This allows the local models to fit their respective local datasets while leveraging information from other clients to improve sample efficiency. In the system, there are $n$ clients and one server. Each client stores a local graph dataset $G_i=(V_i,E_i, \mathbf{X}_i)$, where $V_i$ represents the node set, $E_i$ represents the edge set, $\mathbf{X}_i \in \mathbb{R}^{|V_i| \times D}$ contains the node features, and $[n]$ denotes the set of positive integers from $1$ to $n$. Within the system, the $n$ clients train $n$ GNN models $f(G_i;\mathbf{W}_i), i \in [n]$, with the same structure but different parameters ${\mathbf{W}i}, {i \in [n]}$. Additionally, we assume that the prediction function $f = g \circ h$ is composed of an $L$-layer message passing GNN module $h$ with a hidden dimension of $d$, and a READOUT module $g$ that maps the node embedding extracted by $h$ to downstream task predictions. The goal of a personalized federated GNN framework is formulated as 
\begin{equation}
        \min_{\mathbf{W}_i \in \Omega_i, i \in [n]} \sum_{i=1}^n \mathcal{L}(f(G_i; \mathbf{W}_i), \mathbf{y}_i)
    \label{eq:goal}
\end{equation}
where $\mathcal{L}$ is the loss function, and $\mathbf{y}_i$ contains all the labels belonging to client $i$.
{To encourage collaboration among clients, the parameter spaces $\Omega_i, i \in [n]$, are usually assumed to be related \cite{duan2022adaptive}. The precise structure of this relation is sometimes implicit and instead reflected in the optimization procedure \cite{chen2022personalized}. Under the graph representation learning setup, we require $\Omega_i$ to capture the relatedness between corresponding tasks, taking into account the topological structure of the graph \cite{girvan2002community}, as well as feature and label information. Additionally, we impose an extra constraint to ensure that local models do not deviate significantly from each other. This is achieved by adding a proximal regularization term that prevents overfitting to local data, which has been shown to benefit many federated learning procedures \cite{li2020federated, li2021ditto}.}

\section{Design\label{sec:design}}
\subsection{Overview}

In this section, we provide a detailed introduction to our personalized federated graph learning framework, which aims to address two major problems:
\begin{itemize}[leftmargin=*] 
\item How to extract task features from the local dataset $G_i$ and local model parameters $\mathbf{W}_i$? 
\item How to relate local tasks with each other using the task features to aggregate $\mathbf{W}_i$? 
\end{itemize}

To address the first question, we propose a feature extractor based on dataset distillation, as illustrated in Fig. \ref{fig:tfe}, that captures all the information within the local model. The feature extractor generates a small graph in each round based on the current local model weights. To mitigate graph heterogeneity, the server distributes a common initial graph to all clients, preventing significant deviations among the distilled graphs.

To address the second question, we draw insights from recent advancements in kernel formulations of self-attention \cite{tsai2019transformer,chen2022structure}. We view the personalized aggregation process as an attentive mechanism operating on the \emph{collaboration network} among clients. We observe that several contemporary aggregation schemes overlook the global task relatedness. Leveraging tools from kernel theory, we derive a refined aggregation scheme based on exponential kernel construction that effectively incorporates global information, as shown in Fig. \ref{fig:tr}.

\subsection{Task Feature Extractor} 
\subsubsection{Motivation}
{It is well known in the theory of multi-task learning \cite{hanneke2022no, duan2022adaptive} that correct specifications of task relatedness may fundamentally impact the model performance, which has also been recently discovered in PFL \cite{pmlr-v195-zhao23b}. In hindsight, the ideal characterization of a (local) graph representation learning task might be either the \emph{joint} distribution of the local graph, feature and label variables; or the Bayes optimal learner dervied from the joint distribution \cite{pmlr-v195-zhao23b}. However, none of this information are available during FL, and various surrogates have been proposed in the context of graph PFL that extracts \emph{task features} from the (local) empirical distribution and the learned model.

The most ad-hoc solution is to use weights \cite{long2023multi, ye2023personalized} and gradients \cite{sattler2020clustered}, which are typically high-dimensional (random) vectors. However, computing their relations using metrics like Euclidean or cosine similarity can be unreliable due to the curse of dimensionality phenomenon \cite{hastie2009elements}, as empirically validated in \cite{baek2022personalized}. As a notable state-of-the-art model, FedPUB \cite{baek2022personalized} uses low-dimensional graph embeddings that are produced by passing a shared random graph between clients. However, since the embedding computation only involves message passing GNN layers, the resulting embeddings are \emph{incomplete}, as they fail to represent the READOUT layer that follows these GNN layers. The READOUT layer encodes label-related information.}
This limitation is significant when two datasets share similar graph distributions but have divergent label distributions. This can result in two local models with similar GNN layer weights but different READOUT layer weights. In such cases, the embeddings, which are outputs of similar GNN layers, cannot distinguish between the two datasets. We conducted a small experiment to validate this point. We visualized the embeddings for GCN layers trained on two datasets with similar graph distributions but divergent label distributions in Fig. \ref{fig:embed_vis_dis}. To be more specific, we trained node embeddings on the original Cora graph \cite{yang2016revisiting} and a revised Cora graph in which the label $y_v\in [C]$ of any node $v$ is modified to $C+1-y_v$, where $C=7$ is the number of classes in Cora. The results show that the two embedding spaces are similar, as vertices belonging to the same community are located in similar positions in the space, as shown in Fig. \ref{fig:embed_vis_dis}.

To address the challenges encountered in PFL frameworks, we leverage graph dataset distillation, a method that simultaneously compresses the local data distribution and the learned local model into a size-controlled small dataset that is comparable across all client tasks. In the following sections, we will introduce dataset distillation and explain how we incorporate it into our framework.

\subsubsection{Dataset Distillation}
Dataset distillation \cite{wang2018dataset} (DD) is a centralized knowledge distillation method that aims to distill large datasets into smaller ones. For client $i$, a distilled dataset $(G_i^s, \mathbf{y}_i^s)$ is defined such that a neural network model trained on $G_i^s$ can achieve comparable performance to the one trained on the original dataset $G_i$, as formulated in \eqref{eq:dd}.
\begin{equation}
    \label{eq:dd}
    \begin{aligned}
    \min_{G_i^s, \mathbf{y}_i^s} \mathcal{L}(f(G_i;\mathbf{W}_i^s),\mathbf{y}_i) \quad
    \text{s.t. }  \mathbf{W}_i^s = \min_{\mathbf{W}_i'} \mathcal{L}(f(G_i^s;\mathbf{W}_i'),\mathbf{y}_i^s) 
    \end{aligned}
\end{equation}
{
According to previous studies \cite{jin2022graph}, many datasets can be distilled into condensed ones with sizes that are only around $1\%$ of the original dataset while still preserving model performance. Moreover, it has been empirically reported that distilled datasets offer good privacy protection \cite{dong2022privacy}.

Based on this observation, we propose using \emph{statistics of the distilled local datasets} as features that describe local tasks and obtain task-relatedness by evaluating the similarities between distilled dataset characteristics. As a straightforward adaptation of vanilla DD to federated settings, we may conduct isolated distillation steps \emph{before} the federated training and fix the estimated task-relatedness during federated training. This strategy could be implemented using off-the-shelf DD algorithms on graphs \cite{jin2022graph, jin2022condensing}. However, the quality of the distilled local datasets may be affected by (local) sample quality and quantity. Since PFL approaches typically improve local performance, we propose a refinement of the aforementioned \emph{static distillation} strategy that allows clients to distill their local datasets \emph{during} the federated training process, resulting in a series of distilled datasets ${G_i^{s, t}, \mathbf{y}_i^{s, t}}{t \in [T]}$ for each client $i$, with its corresponding distillation objective at round $t$ being:
\begin{equation}
     \label{eq:dd_fgl_2}
    \begin{aligned}
    &\min_{G_i^{s,t}, \mathbf{y}_i^{s,t}} \mathcal{L}(f(G_i^{s,t};\mathbf{W}_i^t),\mathbf{y}_i^{s,t}). 
    \end{aligned}
\end{equation}
Apart from its capability to adapt to the federated learning procedure, the objective \eqref{eq:dd_fgl_2} is computationally more efficient than the vanilla DD objective \eqref{eq:dd} as it avoids the bi-level optimization problem, which is difficult to solve \cite{wang2018dataset}. Instead, the objective \eqref{eq:dd_fgl_2} leverages the strength of the federated learning process, which usually produces performative intermediate results after a few rounds of aggregations. We refer to \eqref{eq:dd_fgl_2} as a \emph{dynamic distillation} strategy. Next, we present a detailed implementation of the proposed dynamic distillation procedure.
}
\begin{figure}[t]
\centering
\subfloat[Embeddings on Cora with the original labels]{\includegraphics[width=0.22\textwidth]{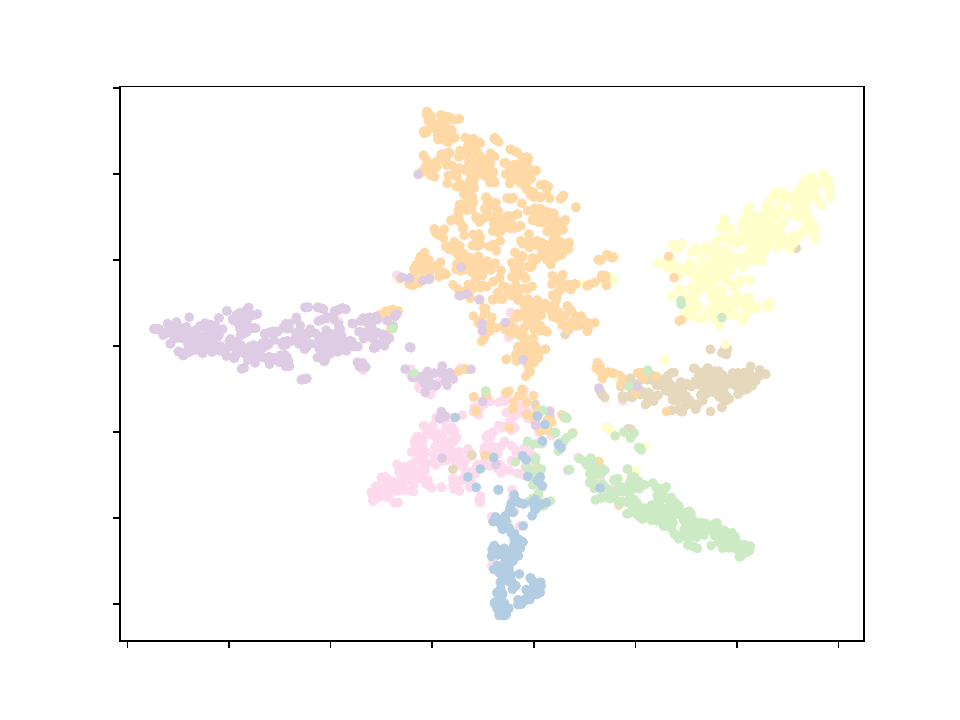}
}
\hfill
\hspace{-0.5cm}
\subfloat[Embeddings on Cora with relabelled nodes]{\includegraphics[width=0.22\textwidth]{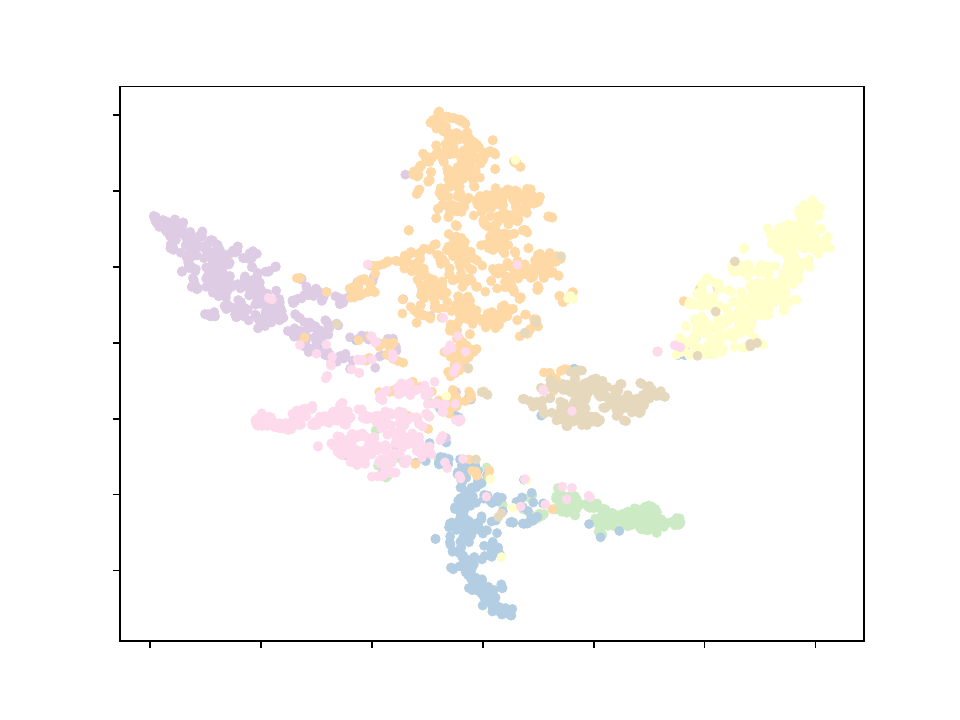}
}
\caption{Embedding spaces (depicted using two leading principle components) trained on two same graphs with divergent labels: vertices belonging to the same community have the same color.}
\label{fig:embed_vis_dis}
\end{figure}
\subsubsection{Implementation}
{
There are two algorithmic goals regarding the implementation of \eqref{eq:dd_fgl_2}: Firstly, the distilled datasets should allow efficient similarity comparisons. Secondly, the problem should be efficiently solved so that the extra computation cost for each client is controllable. Note that both goals are non-trivial since the optimization involves a graph-structured object that is not affected by permutations, resulting in alignment issues when performing similarity computation. The solution is detailed in Algorithm \ref{alg:extractor} in Appendix \ref{sec:alg}. Specifically, at each round $t \in [T]$, the size of the distilled graph across all clients will be fixed at $m \times C$, where $m$ represents the number of representative nodes in each category. The server first initializes node features $\mathbf{X}_0$, with each entry drawn independently from a standard Gaussian distribution $\mathcal{N}(0,1)$. The initial labels $\mathbf{y}_0$ are set to ensure that there are $m$ nodes belonging to each category. The tuple $(\mathbf{X}_0, \mathbf{y}_0)$ is broadcast to each client as the initial value of their local objectives, while the construction of the distilled graph structure is left to the clients' side to reduce communication cost. This common initialization technique alleviates the alignment issue between distilled graphs.
}


After each client receives the initial features $\mathbf{X}_0$ and labels $\mathbf{y}_0$, it begins to update the features and labels. Since directly optimizing the graph structure (i.e., among the space of possible binary matrices) is computationally intractable, we use the following simple generative model that describes the relationship between node features and edge adjacency for the distilled graph:
For a pair of nodes (regarding the distilled graph) $u$ and $v$ with features $\mathbf{x}_u^s$ and $\mathbf{x}_v^s$,
the probability of them being adjacent is given by 
\begin{equation}\label{Eq.cons_g}
    \mathbb{P}[\mathbf{A}^s_{uv} = 1] = \dfrac{e^{\langle \mathbf{x}^s_u, \mathbf{x}^s_v\rangle - \gamma}}{1 + e^{\langle \mathbf{x}^s_u, \mathbf{x}^s_v\rangle - \gamma}},
\end{equation}
where $\gamma > 0$ is a hyperparameter that controls edge sparsity. 
\footnote{This construction is inherently \emph{homophilic}. In principle, one could propose more sophisticated generative mechanisms with learnable parameters, but this may increase the computational cost of distillation. Experimentally, we have found this simple construction to be quite effective.}
Construction of the distilled graph involves sampling from the above distribution, which is not differentiable. Hence we adopt the Gumbel-softmax mechanism \cite{jang2016categorical,maddison2016concrete} to generate approximate yet differentiable samples. In particular, for each $u, v$, we first draw two independent samples $\omega$ and $\omega’$ from the standard Gumbel distribution. Next, we compute the following approximation: 
\begin{equation}
\label{Eq.cons_a} 
p_{uv}(\tau_g)= \dfrac{e^{\left(\langle \mathbf{x}^s_u, \mathbf{x}^s_v\rangle - \gamma + \omega - \omega^\prime\right) / \tau_g}}{1 + e^{\left(\langle \mathbf{x}^s_u, \mathbf{x}^s_v\rangle - \gamma + \omega - \omega^\prime\right) / \tau_g}}, 
\end{equation} 
which adopts a distribution limit $\lim_{\tau_g \rightarrow 0} p_{uv}(\tau_g) \overset{d}{=} \mathbf{A}^s_{uv}$. In practice, we use the straight-through trick \cite{jang2016categorical} to obtain discrete samples from \eqref{Eq.cons_a} while allowing smooth differentiation. We denote the distilled graph as $G^s = (\mathbf{X}^s, \mathbf{P})$ with $\mathbf{P}$ being the (approximated) adjacency matrix, with each entry derived from \eqref{Eq.cons_a}.

We utilize the local model weights to assess how well the distilled dataset fits the model and update $\mathbf{X}^s$ and $\mathbf{y}^s$ based on the same classification loss as each client's local learning objective. In practice, we have found that a few steps of gradient updates suffice for the learning performance. After obtaining the distilled graph, we extract task features ${M_i^t}_{i \in V_c}$ for client $i$ at round $t$ as follows: 
\begin{equation}
\label{eqn:task_feature} 
\mathbf{M}^t_i = \left[\mathbf{X}^{s,t}_i \concat\mathbf{H}^{s,t}_i\right],\quad \mathbf{H}^{s,t}_i := h(\mathbf{G}^{s,t}_i, \mathbf{W}^t_i).
\end{equation}
Note that although the distilled labels are not included in the task feature, the label information is fused into $\mathbf{X}^s$ during the distillation process. We will present an empirical study regarding other potential choices of task feature maps in section \ref{sec:ablation_feature}. 

\subsection{Task relator}
\subsubsection{Motivation}
We represent the estimated relationship among tasks using a \emph{time-varying collaboration network} $G_c^{t}=(V_c, \mathbf{R}^t)$, where $t\in [T]$, $V_c = [n]$, and $\mathbf{R}^t \in \mathbb{R}^{n\times n}$ represents the time-dependent task relation matrix. The entry $r^t_{ij}$ measures the task-relatedness between client $i$ and client $j$, obtained by computing similarities of their corresponding task features $\mathbf{M}^t_i$ and $\mathbf{M}^t_j$. This idea has been adopted in some recent PFL proposals \cite{chen2022personalized, ye2023personalized}. 

 {
 Since the matrix $\mathbf{R}^t$ encodes pairwise relationships among client tasks, it offers great flexibility in defining personalized aggregation protocols. We formulate the protocols as the following expectation:
 }
\begin{align}
    \overline{\mat{W}_i^t} \leftarrow \mathbb{E}_{j \sim q_i}\left[\mat{W}_j^t\right]
\end{align}
where $q_i$ is a client-specific distribution over $[n]$, with a trivial case of uniform distribution that corresponds to the aggregation rule in \fedavg. 
The above formulation is closely connected to the self-attention mechanism \cite{vaswani2017attention}. In particular, inspired by recent developments that generalize self-attention using kernel theory \cite{tsai2019transformer}, we parameterize $q_i$ using a kernel-induced distribution: 
\begin{align} 
q_i[j] = \dfrac{k(i,j)}{\sum_{j^\prime \in [n]}k(i, j^\prime)}, \end{align} 
where $k(\cdot, \cdot)$ is a kernel function.

The most straightforward choice would be the softmax kernel \cite{choromanski2020rethinking} that uses the exponentiated edge weights $k(i, j) = e^{r_{ij}}$. However, this method disregards other weights, resulting in a kernel function that only takes \textit{local connectivity} in the collaboration network into account, overlooking \textit{global connectivity}. We illustrate this point using the example in Fig. \ref{fig:connectivities}, where a collaboration network with three vertices has weighted links of $r_{12}< r_{23}=r_{13}$. If we directly use $r_{ij}$ and normalize them, the average local model weights $\overline{\mathbf{W}_1^t}$ will be very close to $\mathbf{W}_3^t$ and far from $\mathbf{W}^t_2$. However, the relation between node 1 and 2 is much stronger than what the quantity $r_{12}$ indicates, as they are also linked by another two-hop path comprising two heavily-weighted edges $(1,3)$ and $(2,3)$.
A recent work \cite{chen2022personalized} attempts to capture information beyond local task pairs by incorporating a GNN-like mechanism over a sparsified collaboration network. This approach aims to integrate more information through a few rounds of message passing. While the approach in \cite{chen2022personalized} extends the scope of similarity evaluation, it still operates in a \emph{local sense} due to the finiteness of message passing rounds and the inherent limitation of oversmoothing phenomenon.
To address this limitation, our framework proposes a novel kernel function that incorporates all the global connectivity. Specifically, our kernel extracts connectivity at hops from 1 to infinity while favoring connectivity with fewer hops.

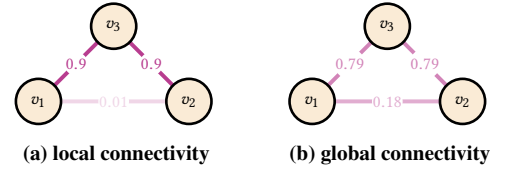
\begin{figure}[!h]
    \centering
    \begin{subfigure}{0.2\textwidth}
        \centering
        \begin{tikzpicture}
    \Vertex[label = $ v_1 $, color=antiquewhite, x=0, y=0]{v1},
    \Vertex[label = $ v_2 $, color=antiquewhite, x=2, y=0]{v2},
    \Vertex[label = $ v_3 $, color=antiquewhite, x=1, y=1]{v3}
    \Edge[color=fandango, opacity=.2, label=$ 0.01 $](v1)(v2)
    \Edge[color=fandango, opacity=.95, label=$ 0.9 $](v1)(v3)
    \Edge[color=fandango, opacity=.95, label=$ 0.9 $](v2)(v3)
\end{tikzpicture}
        \caption{local connectivity}
    \end{subfigure}
    \begin{subfigure}{0.2\textwidth}
        \centering
        \begin{tikzpicture}
    \Vertex[label = $ v_1 $, color=antiquewhite, x=0, y=0]{v1},
    \Vertex[label = $ v_2 $, color=antiquewhite, x=2, y=0]{v2},
    \Vertex[label = $ v_3 $, color=antiquewhite, x=1, y=1]{v3}
    \Edge[color=fandango, opacity=.4, label=$ 0.18 $](v1)(v2)
    \Edge[color=fandango, opacity=.6, label=$ 0.79 $](v1)(v3)
    \Edge[color=fandango, opacity=.6, label=$ 0.79 $](v2)(v3)
\end{tikzpicture}
        \caption{global connectivity}
    \end{subfigure}
    \caption{Comparison of local and global connectivity}
    \label{fig:connectivities}
\end{figure}
\vspace{-0.4cm}
\subsubsection{Construction of the task relator}
According to the previous discussions, implementing the task relator involves the design of two modules: A \emph{collaboration graph construction procedure} based on the extracted task features $\{M_i^t\}_{i \in V_c}$ and a \emph{global-connectivity-aware aggregation mechanism}. 

\begin{table*}[th]

\resizebox{\textwidth}{!}{
\small

\begin{tabular}{lccccccccc}
\toprule[2pt]
Dataset & \multicolumn{3}{c}{Cora} & \multicolumn{3}{c}{CiteSeer} & \multicolumn{3}{c}{PubMed} \\
\midrule
\# Clients & 10    & 30    & 50    & 10    & 30    & 50    & 10    & 30    & 50 \\
\midrule
Local & 46.88$\pm$1.23  & 66.45$\pm$0.81  & 70.32$\pm$0.68  & 51.42$\pm$1.75  & 59.06$\pm$1.64  & 61.40$\pm$1.45  & 76.75$\pm$0.20  & 77.46$\pm$0.20  & 76.02$\pm$0.33  \\
\midrule
FedAvg & 49.75$\pm$1.32  & 46.20$\pm$2.67  & 43.48$\pm$3.97  & 54.79$\pm$2.86  & 54.14$\pm$1.41  & 57.52$\pm$1.98  & 78.53$\pm$0.68  & 80.99$\pm$0.26  & 79.53$\pm$0.06  \\
FedProx  & 50.79$\pm$2.00  & 54.72$\pm$5.21  & 62.11$\pm$2.02  & 56.31$\pm$5.81  & 59.41$\pm$0.68  & 63.29$\pm$1.21  & 77.32$\pm$0.88  & 80.99$\pm$0.51  & 79.60$\pm$0.21  \\
\midrule
FedPer  & 52.83$\pm$0.55  & 67.15$\pm$0.85  & 70.27$\pm$0.34  & 57.14$\pm$1.45  & 62.21$\pm$1.80  & 63.26$\pm$1.95  & 79.85$\pm$0.31  & 80.59$\pm$0.06  & 80.28$\pm$0.13  \\
\midrule
FedPub  & 52.58$\pm$1.51  & 67.30$\pm$0.99  & 42.81$\pm$5.70  & 56.06$\pm$2.29  & 62.12$\pm$0.49  & 64.18$\pm$1.88  & 79.70$\pm$0.21  & 80.97$\pm$0.22  & 80.56$\pm$0.23  \\
FedSage & 49.25$\pm$0.50  & 59.42$\pm$1.03  & 59.99$\pm$0.23  & 55.54$\pm$6.95  & 55.63$\pm$7.00  & 62.73$\pm$1.09  & 77.87$\pm$0.50  & 80.97$\pm$0.24  & 79.36$\pm$0.73  \\
GCFL  & 49.52$\pm$0.33  & 46.78$\pm$4.32  & 45.55$\pm$6.03  & 56.03$\pm$2.04  & 53.91$\pm$0.38  & 56.43$\pm$0.41  & 76.03$\pm$2.04  & 79.58$\pm$0.13  & 78.68$\pm$0.15  \\
FedStar & 43.09$\pm$0.72  & 61.60$\pm$0.30  & 67.77$\pm$1.25  & 46.45$\pm$0.17  & 54.78$\pm$2.12  & 58.96$\pm$1.81  & 75.45$\pm$0.14  & 76.45$\pm$0.43  & 74.71$\pm$0.52  \\
\midrule
Ours  & \boldmath{}\textbf{53.26$\pm$1.42 }\unboldmath{} & \boldmath{}\textbf{67.88$\pm$1.09 }\unboldmath{} & \boldmath{}\textbf{70.41$\pm$0.51 }\unboldmath{} & \boldmath{}\textbf{58.19$\pm$1.82 }\unboldmath{} & \boldmath{}\textbf{62.30$\pm$1.33 }\unboldmath{} & \boldmath{}\textbf{64.58$\pm$0.55 }\unboldmath{} & \boldmath{}\textbf{79.90$\pm$0.53 }\unboldmath{} & \boldmath{}\textbf{81.65$\pm$0.34 }\unboldmath{} & \boldmath{}\textbf{80.82$\pm$0.20 }\unboldmath{} \\
\midrule[2pt]
Dataset & \multicolumn{3}{c}{Amazon Photo} & \multicolumn{3}{c}{Amazon Computers} & \multicolumn{3}{c}{Ogbn Arxiv} \\
\midrule
\# Clients & 10    & 30    & 50    & 10    & 30    & 50    & 10    & 30    & 50 \\
\midrule
Local & 46.57$\pm$0.15  & 69.25$\pm$0.25  & 79.42$\pm$0.34  & 51.82$\pm$0.62  & 65.69$\pm$0.94  & 68.57$\pm$0.35  & 34.76$\pm$0.50  & 46.98$\pm$0.18  & 47.45$\pm$0.19  \\
\midrule
FedAvg & 43.10$\pm$2.68  & 44.75$\pm$4.82  & 46.38$\pm$1.07  & 44.45$\pm$0.26  & 52.93$\pm$1.05  & 53.91$\pm$0.57  & 41.40$\pm$0.46  & 44.22$\pm$0.80  & 43.74$\pm$2.60  \\
FedProx  & 43.58$\pm$2.05  & 45.29$\pm$0.53  & 42.76$\pm$5.23  & 42.59$\pm$4.17  & 53.58$\pm$1.56  & 53.91$\pm$0.57  & 41.35$\pm$0.20  & 44.68$\pm$0.62  & 47.02$\pm$0.52  \\
\midrule
FedPer  & 52.20$\pm$0.99  & 71.76$\pm$0.57  & 81.65$\pm$0.06  & 55.04$\pm$0.33  & 67.55$\pm$1.42  & 68.79$\pm$0.51  & 38.77$\pm$0.44  & 47.46$\pm$0.18  & 50.51$\pm$0.15  \\
\midrule
FedPub  & 45.69$\pm$2.10  & 64.50$\pm$0.48  & 76.58$\pm$0.85  & 50.15$\pm$1.57  & 60.81$\pm$0.52  & 63.82$\pm$0.62  & 42.18$\pm$0.36  & 50.58$\pm$0.21  & 51.11$\pm$0.56  \\
FedSage & 47.79$\pm$0.76  & 58.26$\pm$2.35  & 58.99$\pm$1.58  & 47.98$\pm$0.84  & 56.82$\pm$0.43  & 63.13$\pm$0.86  & 42.18$\pm$0.11  & 45.43$\pm$0.40  & 46.08$\pm$0.27  \\
GCFL  & 46.93$\pm$0.55  & 48.95$\pm$0.48  & 40.76$\pm$8.08  & 56.02$\pm$1.04  & 58.38$\pm$2.16  & 52.97$\pm$1.74  & 41.32$\pm$0.23  & 46.26$\pm$1.12  & 46.16$\pm$0.34  \\
FedStar & 39.05$\pm$0.43  & 64.51$\pm$0.28  & 77.35$\pm$0.69  & 47.93$\pm$0.82  & 61.67$\pm$1.01  & 64.34$\pm$0.71  & 40.92$\pm$0.13  & 44.28$\pm$0.19  & 45.27$\pm$0.35  \\
\midrule
Ours  & \boldmath{}\textbf{56.07$\pm$1.44 }\unboldmath{} & \boldmath{}\textbf{72.15$\pm$0.49 }\unboldmath{} & \boldmath{}\textbf{81.91$\pm$0.29}\unboldmath{} & \boldmath{}\textbf{56.56$\pm$0.15 }\unboldmath{} & \boldmath{}\textbf{68.43$\pm$0.28}\unboldmath{} & \boldmath{}\textbf{69.11$\pm$1.07}\unboldmath{} & \boldmath{}\textbf{43.02$\pm$0.04 }\unboldmath{} & \boldmath{}\textbf{48.32$\pm$0.14}\unboldmath{} & \boldmath{}\textbf{52.38$\pm$0.12 }\unboldmath{} \\
\bottomrule[2pt]
\end{tabular}%
}

\caption{Node classification performance (\%) on overlapping datasets\label{tab:perf_over}}
\end{table*}
\begin{table*}[th]
\resizebox{\textwidth}{!}{
\small
\begin{tabular}{lccccccccc}
\toprule[2pt]
Dataset & \multicolumn{3}{c}{Cora} & \multicolumn{3}{c}{CiteSeer} & \multicolumn{3}{c}{PubMed} \\
\midrule
\# Clients & 5     & 10    & 20    & 5     & 10    & 20    & 5     & 10    & 20 \\
\midrule
Local & 80.44$\pm$1.77  & 79.58$\pm$1.07  & 79.46$\pm$0.51  & 71.28$\pm$1.32  & 68.93$\pm$0.79  & 70.49$\pm$0.88  & 84.98$\pm$0.67  & 83.34$\pm$0.42  & 82.92$\pm$0.41  \\
\midrule
FedAvg & 72.91$\pm$6.59  & 69.23$\pm$1.03  & 47.79$\pm$2.61  & 72.43$\pm$0.99  & 70.33$\pm$1.37  & 67.18$\pm$1.17  & 82.02$\pm$0.15  & 83.17$\pm$1.35  & 76.50$\pm$0.20  \\
FedProx  & 63.96$\pm$3.07  & 71.39$\pm$4.16  & 70.88$\pm$5.90  & 73.63$\pm$0.85  & 42.86$\pm$2.59  & 42.31$\pm$3.18  & 83.99$\pm$0.17  & 83.57$\pm$0.09  & 83.93$\pm$0.89  \\
\midrule
FedPer  & 81.37$\pm$1.58  & 76.73$\pm$0.95  & 77.24$\pm$1.62  & 70.45$\pm$2.00  & 67.14$\pm$6.95  & 71.13$\pm$0.76  & 85.59$\pm$0.18  & 85.42$\pm$0.12  & 83.75$\pm$0.14  \\
\midrule
FedPub  & 82.33$\pm$1.46  & 78.27$\pm$1.46  & 79.15$\pm$1.08  & 74.11$\pm$1.58  & 72.12$\pm$1.83  & 68.16$\pm$1.41  & 86.22$\pm$0.21  & 85.58$\pm$0.31  & 84.79$\pm$0.46  \\
FedSage & 72.07$\pm$0.36  & 69.66$\pm$0.27  & 59.28$\pm$0.38  & 70.64$\pm$3.04  & 65.54$\pm$6.95  & 63.02$\pm$1.49  & 84.64$\pm$0.60  & 83.39$\pm$1.29  & 84.92$\pm$0.45  \\
GCFL  & 79.91$\pm$1.93  & 73.25$\pm$4.39  & 76.37$\pm$1.82  & 71.37$\pm$2.54  & 67.58$\pm$0.61  & 63.54$\pm$3.34  & 84.24$\pm$0.57  & 83.47$\pm$0.29  & 83.72$\pm$0.47  \\
FedStar & 79.33$\pm$0.69  & 78.26$\pm$0.22  & 80.40$\pm$0.30  & 69.47$\pm$1.77  & 70.25$\pm$1.26  & 68.50$\pm$0.68  & 81.96$\pm$0.96  & 81.39$\pm$0.17  & 80.15$\pm$0.66  \\
\midrule
Ours  & \boldmath{}\textbf{83.37$\pm$1.59 }\unboldmath{} & \boldmath{}\textbf{80.06$\pm$1.27 }\unboldmath{} & \boldmath{}\textbf{81.17$\pm$0.63 }\unboldmath{} & \boldmath{}\textbf{75.25$\pm$1.38 }\unboldmath{} & \boldmath{}\textbf{74.18$\pm$1.14 }\unboldmath{} & \boldmath{}\textbf{71.17$\pm$1.69 }\unboldmath{} & \boldmath{}\textbf{87.05$\pm$0.19 }\unboldmath{} & \boldmath{}\textbf{86.53$\pm$0.73 }\unboldmath{} & \boldmath{}\textbf{86.38$\pm$0.30 }\unboldmath{} \\
\midrule[2pt]
Dataset & \multicolumn{3}{c}{Amazon Photo} & \multicolumn{3}{c}{Amazon Computers} & \multicolumn{3}{c}{Ogbn Arxiv} \\
\midrule
\# Clients & 5     & 10    & 20    & 5     & 10    & 20    & 5     & 10    & 20 \\
\midrule
Local & 77.97$\pm$0.29  & 86.14$\pm$1.05  & 86.37$\pm$0.17  & 65.90$\pm$0.29  & 74.41$\pm$1.51  & 81.81$\pm$0.50  & 56.93$\pm$0.89  & 56.54$\pm$0.37  & 57.79$\pm$0.89  \\
\midrule
FedAvg & 53.49$\pm$5.87 & 45.82$\pm$1.88  & 35.15$\pm$1.03  & 46.03$\pm$1.93 & 39.04$\pm$3.68  & 43.74$\pm$8.15  & 55.84$\pm$0.88  & 61.02$\pm$0.32  & 59.30$\pm$0.18  \\
FedProx  & 71.08$\pm$3.11  & 56.78$\pm$4.31  & 44.61$\pm$5.89  & 37.72$\pm$0.94  & 36.44$\pm$0.35  & 36.89$\pm$0.27  & 62.05$\pm$1.10  & 61.77$\pm$0.78 & 57.79$\pm$0.26  \\
\midrule
FedPer  & 68.19$\pm$1.68  & 77.15$\pm$0.14  & 78.96$\pm$0.68  & 64.30$\pm$0.34  & 64.47$\pm$0.20  & 70.44$\pm$0.57  & 61.57$\pm$0.50  & 61.52$\pm$0.37  & 62.73$\pm$0.26  \\
\midrule
FedPub  & 86.76$\pm$1.71  & 87.80$\pm$2.44  & 88.72$\pm$3.09  & 68.65$\pm$2.53  & 77.02$\pm$0.87  & 80.71$\pm$0.79  & 67.50$\pm$0.32  & 66.80$\pm$0.32  & 62.11$\pm$0.56  \\
FedSage & 51.28$\pm$7.30  & 51.68$\pm$7.28  & 51.39$\pm$7.22  & 42.88$\pm$5.23  & 50.41$\pm$7.84  & 57.06$\pm$0.42  & 58.63$\pm$1.29  & 61.65$\pm$0.45  & 54.86$\pm$1.77  \\
GCFL  & 68.17$\pm$8.37  & 82.74$\pm$3.15  & 87.55$\pm$2.28  & 55.36$\pm$2.50 & 72.53$\pm$1.38  & 82.87$\pm$1.83  & 59.75$\pm$3.46  & 63.63$\pm$0.36  & 55.35$\pm$4.58  \\
FedStar & 85.67$\pm$0.31  & 86.85$\pm$0.09  & 87.60$\pm$0.68  & 71.88$\pm$1.70  & 78.81$\pm$1.41  & 83.42$\pm$0.73  & 58.96$\pm$0.82  & 60.77$\pm$0.46  & 61.36$\pm$0.14  \\
\midrule
Ours  & \boldmath{}\textbf{89.16$\pm$0.04 }\unboldmath{} & \boldmath{}\textbf{88.83$\pm$0.85 }\unboldmath{} & \boldmath{}\textbf{89.53$\pm$0.73 }\unboldmath{} & \boldmath{}\textbf{72.75$\pm$2.16 }\unboldmath{} & \boldmath{}\textbf{82.68$\pm$0.73 }\unboldmath{} & \boldmath{}\textbf{84.23$\pm$0.41 }\unboldmath{} & \boldmath{}\textbf{68.52$\pm$0.14 }\unboldmath{} & \boldmath{}\textbf{67.87$\pm$0.27 }\unboldmath{} & \boldmath{}\textbf{65.27$\pm$0.51 }\unboldmath{} \\
\bottomrule[2pt]
\end{tabular}}

\caption{Node classification performance (\%) on non-overlapping datasets\label{tab:perf_non}}
\end{table*}
To form the collaboration graph, we use the following feature-wise average correlation as the pairwise task-relatedness:
\begin{align}
    r_{ij}^t = \frac{1}{d+D} \sum_{k = 1}^{d+D} \textsf{corr}\left(\mathbf{M}_i^t[:, k], \mathbf{M}_j^t[:, k]\right),
    \label{eq:r}
\end{align}
where \textsf{corr} stands for Pearson's correlation coefficient\cite{lehmann2006theory}.
 Next we discuss the construction of the global-connectivity-aware aggregation mechanism.
 {Inspired by the property of exponential kernels \cite{kondor2002diffusion} that translates local structure into global ones, we use an element-wise exponentiated matrix exponential with two temperature parameters $\tau$ and $\tau_s$: \begin{align}
 \label{Eq.kernel} 
 k(i, j) = e^{\tau_ss_{ij}}, \mat{S} = \{ s_{ij}\}_{i \in V, j \in V}= e^{\tau \mat{R}}=\sum^\infty_{\kappa=1}\frac{1}{\kappa!}(\tau \mat{R})^\kappa. 
 \end{align}
From the right hand side of \eqref{Eq.kernel}, we may interpret $s_{i,j}$ as encoding the relatedness of client $i$ and $j$ via conducting infinite rounds of message passing, thereby reflecting their \textit{global connectivity} structure. 
Then, $k(i,j)$ scales the global connectivity into $\mathbb{R}^+$ range for further normalization into the client-specific distribution $q_i$ over $[n]$. The parameter $\tau$ is used to control the level of personalization, where $\tau \rightarrow 0$ indicates no personalization (\fedavg) and $\tau \rightarrow \infty$ indicates local training. The parameter $\tau_s$ strikes a balance between the contribution of local and global information. It is worth noting that there are other notions of global connectivity measures, such as effective resistance \cite{black2023understanding}, which are also applicable. However, here we stick to the matrix exponential due to its loose requirements of arbitrary local similarity $\mat{R}$. We will prove in appendix \ref{sec:proof} that the function $k$ in \eqref{Eq.kernel} is a valid kernel over the domain $[n]$.

\subsection{Complexity considerations}
In comparison to vanilla \fedavg, the proposed framework requires additional local computations, a server-side matrix exponential operation, as well as extra communication costs. Let us briefly discuss the complexity of these procedures.
Firstly, the dataset distillation procedure operates on small-scale graphs, making the total computation cost negligible compared to local training.
Secondly, the matrix exponential operation upon $\mat{R}$ has a time complexity of $O(n^3)$. This complexity is controllable in practice since the number of clients, $n$, is typically small or moderate.
Finally, the extra communication cost per client per aggregation step depends on the formulation of the task feature map. According to \eqref{eqn:task_feature}, the extra communication cost is  $O(bmC(d+D))$, where $b$ is the number of bits required to represent a floating-point number. This cost is comparable to the communication cost of a GNN. It is noteworthy that since GNN models are typically shallow, the communication cost of parameters is often dominated by the computation cost of local training. Moreover, the communication cost can be further reduced if a more compact task feature, such as the distilled graph embedding, is used. We will empirically investigate such alternatives in section \ref{sec:ablation_feature}.

\section{Experiments\label{sec:ex}}
This section presents the empirical analysis of our framework, which includes a performance comparison, convergence analysis, and multiple ablation studies. 
\subsection{Experiment Setup}
\subsubsection{Datasets}
We have tested the performance on six different graph datasets of varying scales: Cora, CiteSeer, PubMed, Amazon Computer, Amazon Photo, and Ogbn arxiv \cite{yang2016revisiting,shchur2018pitfalls,NEURIPS2020_fb60d411}.
To split each graph into subgraphs, we employed the Metis graph partition algorithm following the setup in \cite{baek2022personalized}. Each client stores one subgraph from the original graph.
We have conducted a node classification task by sampling the vertices into training, validation and testing vertices according to ratio 0.3, 0.35 and 0.35 before splitting. Appendix \ref{app:dataset} provides detailed descriptive statistics of the datasets. 

\subsubsection{Baselines}
We compared our framework with eight different baselines, categorized into four types:
(1) \textbf{Local}, which serves as the standard baseline without federated learning;
(2) Two traditional FL baselines, including \textbf{FedAvg} \cite{mcmahan2017communication} and \textbf{FedProx} \cite{li2020federated};
(3) One state-of-the-art personalized FL baseline, \textbf{FedPer} \cite{arivazhagan2019federated};
(4) Four state-of-the-art personalized federated GNN baselines, including \textbf{FedPub} \cite{baek2022personalized}, \textbf{FedSage} \cite{zhang2021subgraph}, \textbf{GCFL} \cite{xie2021federated}, and \textbf{FedStar} \cite{tan2023federated}.
For detailed introductions of the baselines, please refer to Appendix \ref{app:base}.

\begin{figure*}[th]
\centering
\subfloat[Cora]{\includegraphics[width=0.166\textwidth]{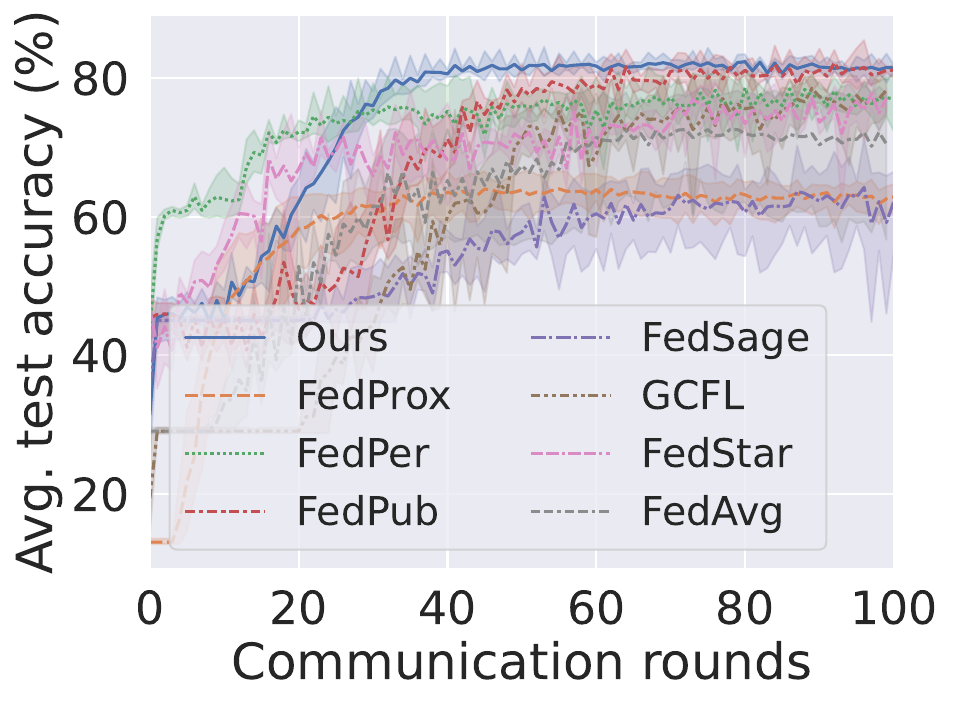}}
\hfill
\hspace{-0.5cm}
\subfloat[CiteSeer]{\includegraphics[width=0.166\textwidth]{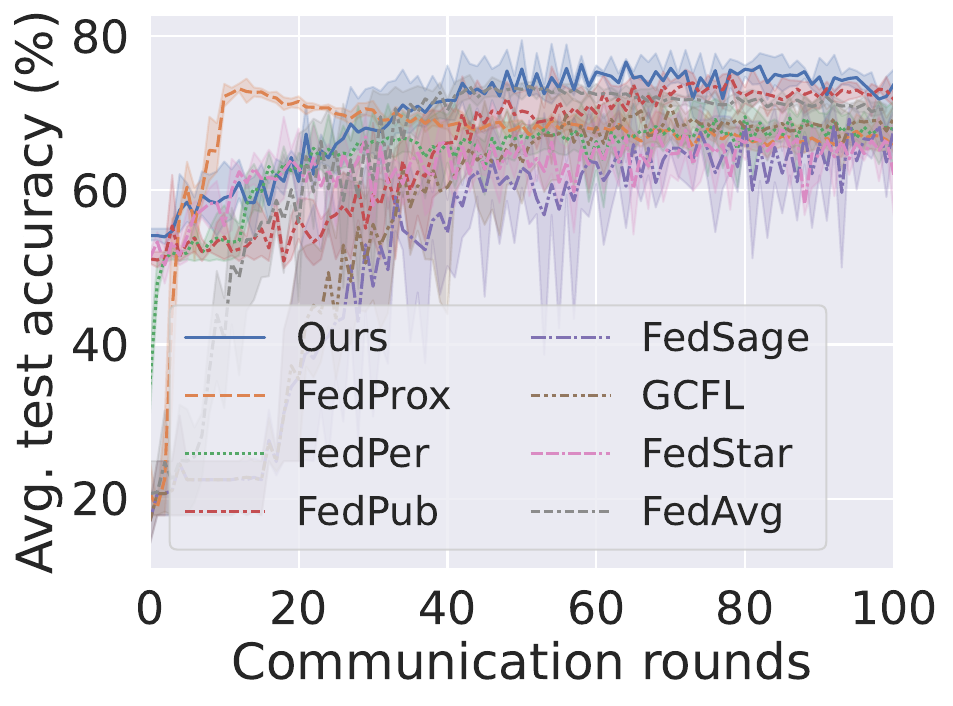}}
\hfill
\hspace{-0.5cm}
\subfloat[PubMed]{\includegraphics[width=0.166\textwidth]{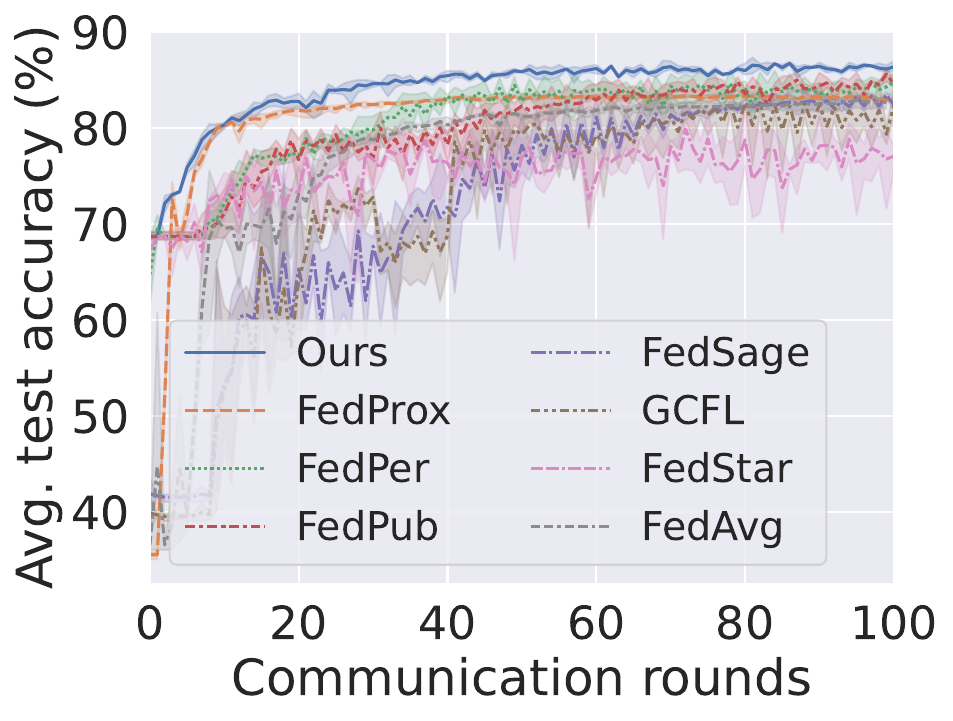}}
\hfill
\hspace{-0.5cm}
\subfloat[Photo]{\includegraphics[width=0.166\textwidth]{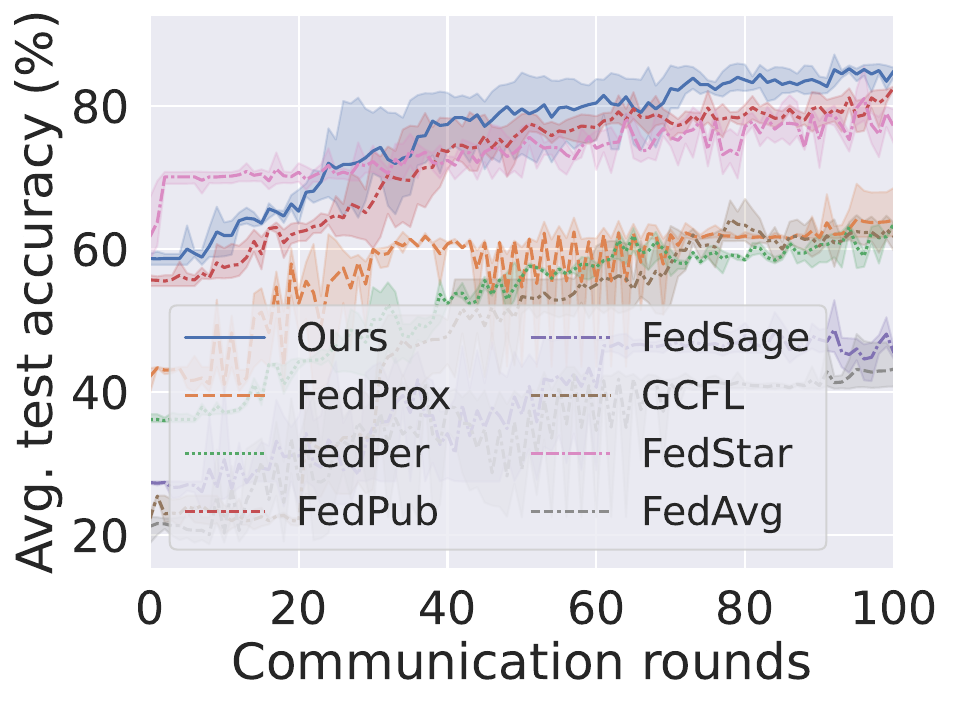}}
\hfill
\hspace{-0.5cm}
\subfloat[Computers]{\includegraphics[width=0.166\textwidth]{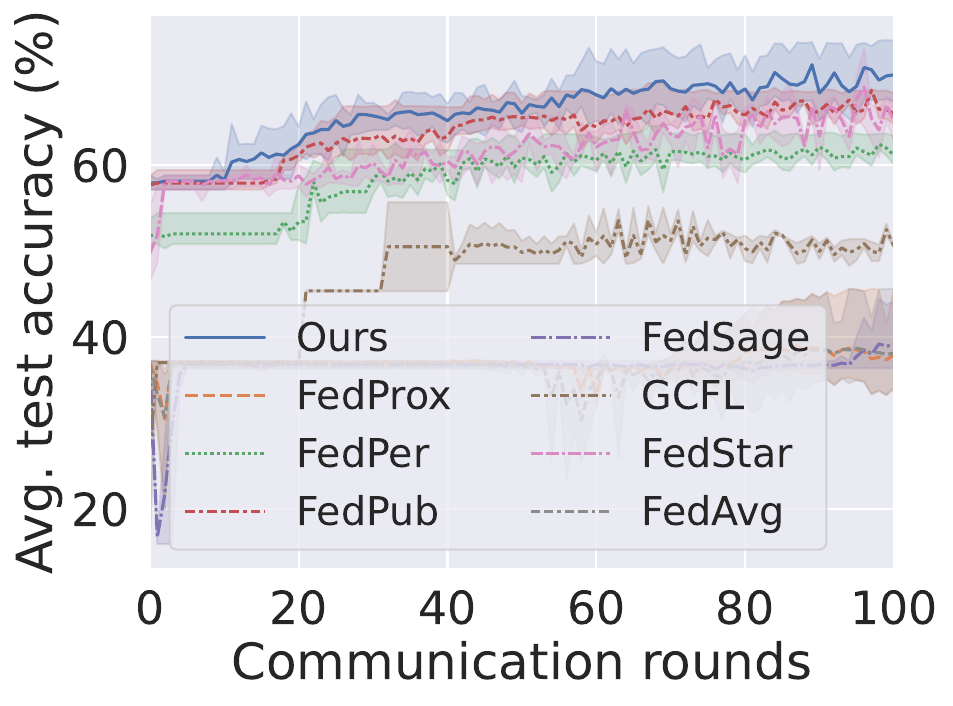}}
\hfill
\hspace{-0.5cm}
\subfloat[Arxiv]{\includegraphics[width=0.166\textwidth]{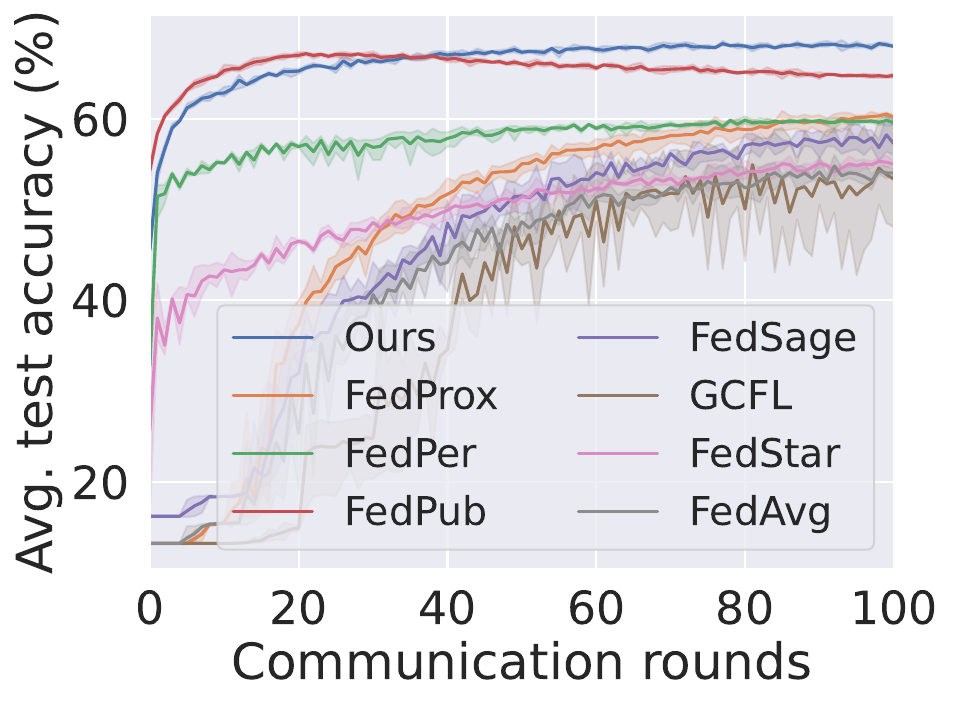}}
\caption{Convergence plot for the non-overlapping setting with 5 clients. We visualize the first 100 communication rounds. }
\label{fig:conv}
\end{figure*}
\begin{figure*}[ht]
\begin{minipage}{0.195\textwidth}
        \centering
    \includegraphics[width=\textwidth]{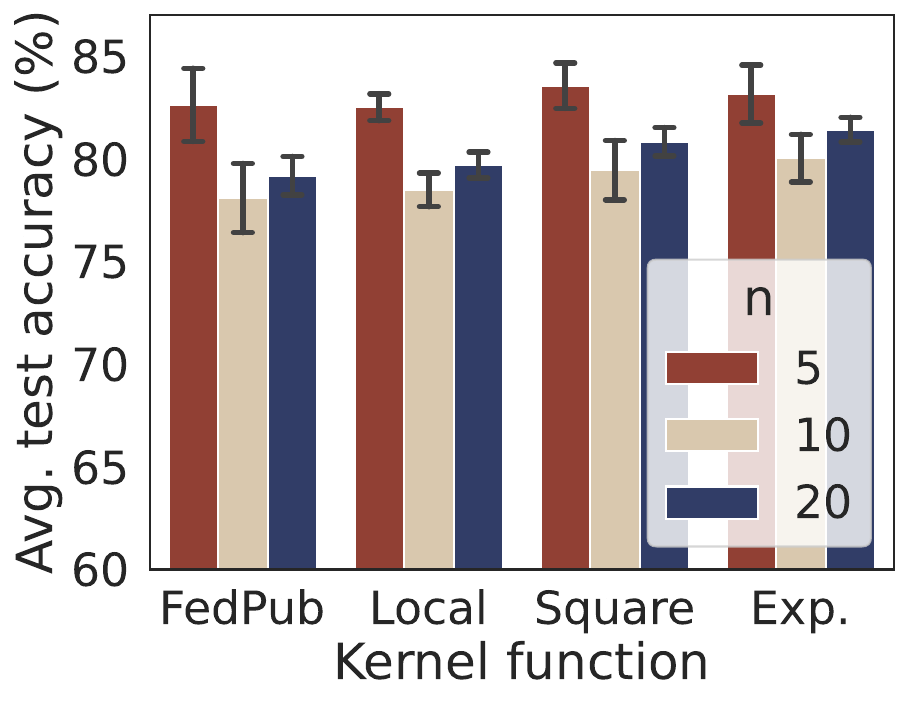}
    \caption{Effects of kernel functions\label{fig:kernel}\newline}
\end{minipage}
\begin{minipage}{0.195\textwidth}
        \centering
    \includegraphics[width=\textwidth]{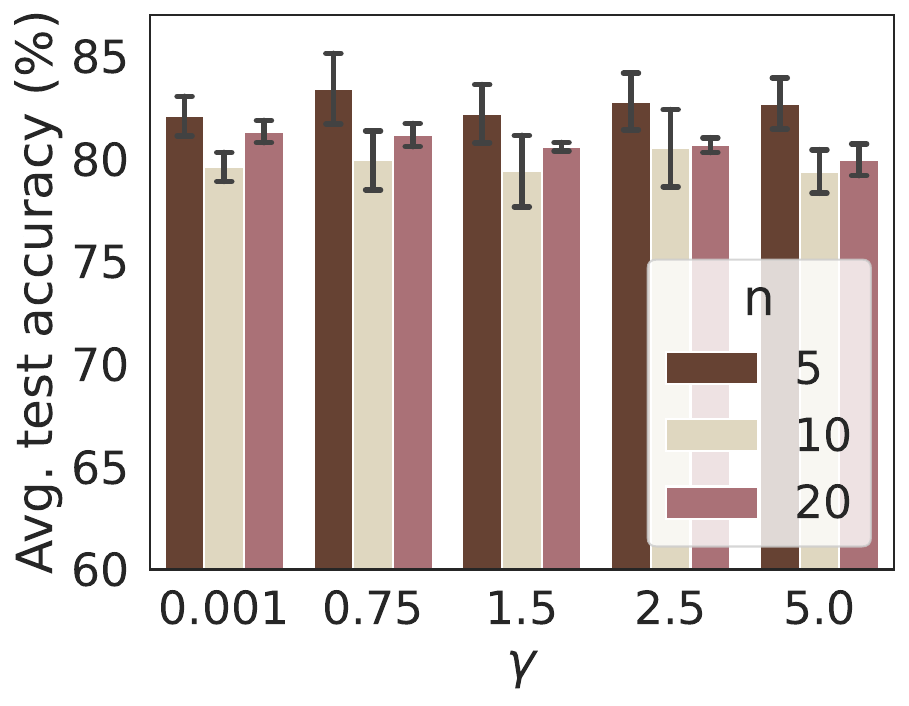}
    \caption{Effects of sparsity control coefficient $\gamma$\label{fig:gamma}\newline}
\end{minipage}
\begin{minipage}{0.195\textwidth}
        \centering
    \includegraphics[width=\textwidth]{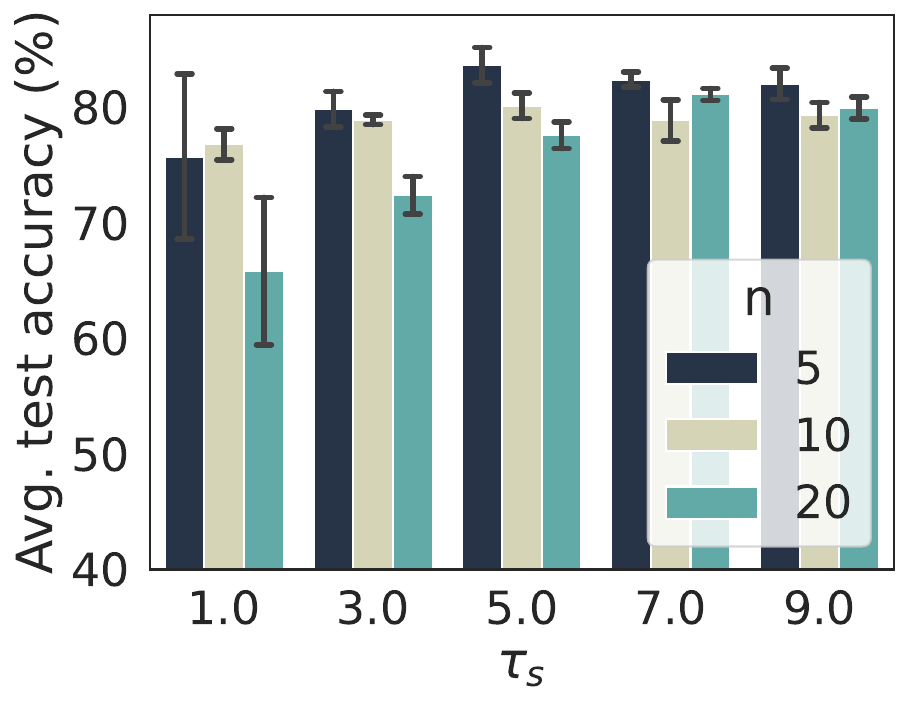}
    \caption{Effects of temperature on element-wise exponential\label{fig:taus} $\tau_s$}
\end{minipage}
\begin{minipage}{0.195\textwidth}
        \centering
    \includegraphics[width=\textwidth]{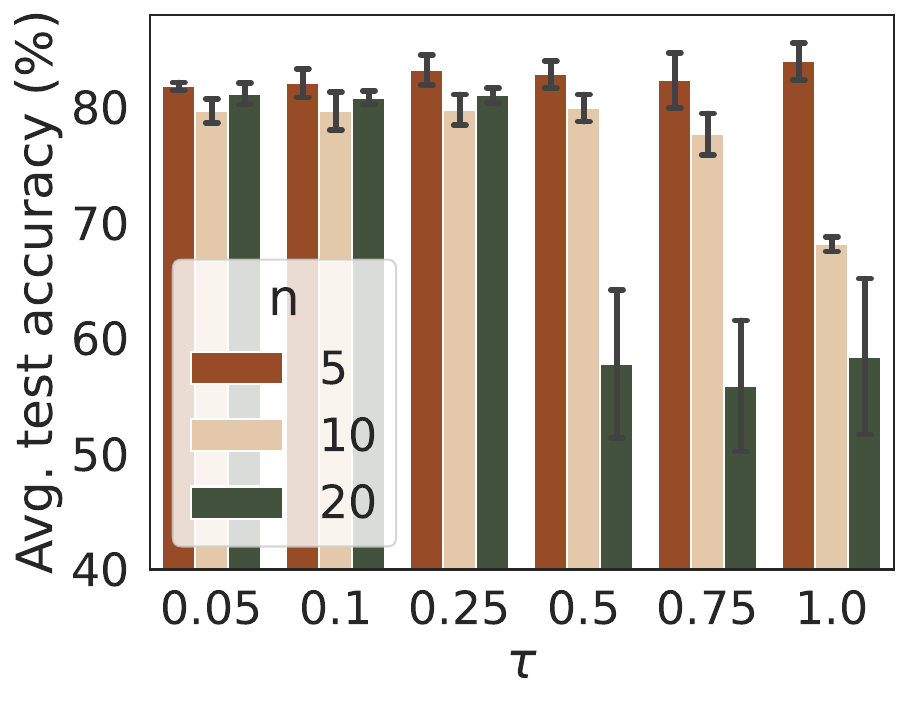}
    \caption{Effects of temperature on local connectivity matrix exponential\label{fig:tau} $\tau$}
\end{minipage}
\begin{minipage}{0.195\textwidth}
    \includegraphics[width=\textwidth]{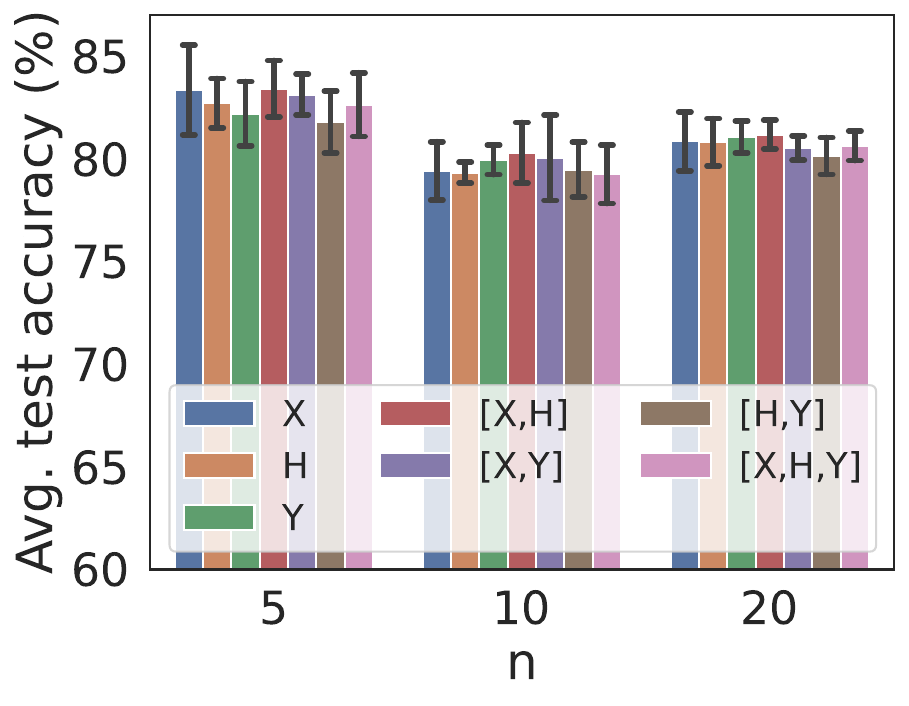}
    \caption{Effects of of task features from distilled graphs    \label{fig:xhy}}
\end{minipage}
\end{figure*}

\subsubsection{Implementation Details} We utilize a two-layer GCN \cite{kipf2016semi} followed by a linear READOUT layer. The dimension of the embeddings is set to 128. To optimize learning, we employ the Adam optimizer with weight decay $10^{-6}$ \cite{kingma2014adam}. The smoothing parameter $\tau_g$ in Gumbel-softmax is set to 1, following the technique in \cite{jang2016categorical}.
To monitor the training progress, we use an early-stop mechanism. If the validation accuracy decreases for 20 consecutive rounds, the FL framework stops immediately.
Each experiment is conducted over 3 runs with different random seeds. Implementation of all methods is done using PyTorch Geometric \cite{fey2019fast} on an NVIDIA Tesla V100 GPU. For further details, please refer to Appendix \ref{app:pc}.

\subsection{Results}
\subsubsection{Performance Comparison} 
We evaluate the node classification performance of various frameworks using six real-world datasets of differing scales. Tables \ref{tab:perf_over} and \ref{tab:perf_non} present the average test accuracy and its standard deviation for overlapping and non-overlapping settings, respectively. Our framework consistently outperforms all other methods across all datasets.
Traditional FL baselines, such as FedAvg and FedProx, which lack local adaptation, are consistently inferior to our Local framework in multiple scenarios. Despite using local task statistics for personalization, the feature extraction schemes of FedPer, FedStar, FedSage, and GCFL are less effective and perform worse than our framework.
FedPub performs worse than our method due to the absence of information contained in local READOUT layers when describing tasks and inefficient exploitation of the global collaboration structure.

\subsubsection{Convergence Analysis}
Fig.\ref{fig:conv} illustrates the convergence behavior of the average test accuracy over the first 100 rounds with 5 clients in the system. It is evident that our proposed framework exhibits a rapid convergence rate towards the highest average test accuracy. This can be attributed to the framework's ability to efficiently capture the local task features and identify pairwise relationships.


\subsubsection{Effects of the proposed task relator}
The primary goal of this study is to investigate the benefits of incorporating the global information. Specifically, we tested with FedGKD along with two local variants obtained via replacing the matrix $S$ in \eqref{Eq.kernel} on the Cora dataset. The \emph{local} variant corresponds to the standard softmax kernel that sets $\mathbf{S} = \mathbf{R}$. The \emph{square} variant corresponds to setting $\mathbf{S}=\mathbf{R}^2$, which can be  understood as performing a two-layer message passing. As shown in Fig.\ref{fig:kernel}, we observe that incorporating global information leads to improved federated learning performance, especially when the number of clients is large. This implies that inter-client information is more complicated, and the proposed task relater provides a more nuanced solution. Furthermore, 
we also compare these variations with FedPub framework and observe that even using local connectivity generated from distilled datasets instead of graph embeddings on a random graph input to relate tasks, our framework outperforms FedPub. This suggests that distilled graphs are more representative than graph embeddings due to their incorporation of READOUT layers.

\subsubsection{Effects of Sparsity-controlling Coefficient $\gamma$}
We conduct an ablation study on Cora to assess the impact of the sparsity control coefficient $\gamma$ in distilled graphs. A larger $\gamma$ generates a more sparse distilled graph. We vary $\gamma$ across $\{10^{-3}, 0.75, 1.5, 2.5, 5 \}$ to test the effect. Our findings show that the optimum value for $\gamma$ in our framework is 0.75. We hypothesize that distilled graphs' density is influenced by the constraint of containing few nodes, as a sparse small graph results in almost no connections. This observation is consistent with results from centralized graph distillation\cite{jin2022graph,jin2022condensing}.

\subsubsection{Effects of Temperature on Element-wise Exponential $\tau_s$}
We investigate the impact of varying the temperature values on the element-wise exponential $\tau_s$ defined in \eqref{Eq.kernel} on Cora dataset. $\tau_s$ is a metric that regulates the influence of the local model weights $W_i^t$ on the aggregated weights $\overline{W_i^t}$. In a federated GNN system with significantly heterogeneous local datasets, a large value of $\tau_s$ is required to achieve optimal performance. This idea is supported by the results presented in Fig.\ref{fig:taus}. In Appendix \ref{app:dataset}, we demonstrate that the number of clients results in more heterogeneous graphs within the system, necessitating a larger value of $\tau_s$ to attain optimal performance in FedGKD.

\subsubsection{Effects of Temperature on Matrix Exponential $\tau$}
We experiment with varying the values of $\tau$ in the exponential of the relation matrix, as defined in \eqref{Eq.kernel} on Cora. $\tau$ is introduced to avoid the singularity of the matrix exponential. As shown in Fig.\ref{fig:tau}, 
a large $\tau$ may results in extremely low rank aggregation weight matrix thereby deteriorating model performance. Therefore, it is essential to set an appropriate value of $\tau$ to guarantee non-singularity.

 \subsubsection{Effects of Task Features from Distilled Graphs}\label{sec:ablation_feature}
We experiment with multiple choices of statistics obtained from distilled 
graphs to compute the pairwise task-relatedness in
 \eqref{eq:r} on Cora dataset. Fig.\ref{fig:xhy} illustrates that choices of statistics are robust to model performance but the concatenation of node feature $\mat{X}$ and embeddings $\mat{H}$ outperforms others slightly. 
It is worth noting that if communication overhead is a major concern, we can further reduce the extra communication cost by transmitting only $\mat{H}$ or even the distilled labels, which incurs only a slight performance degradation.

\subsubsection{Additional experiments} 
We report an ablation study regarding the comparison of static versus dynamic dataset distillation strategy in appendix \ref{sec:static_vs_dynamci}. The results suggest that dynamic strategy is preferred.

\section{Conclusion\label{sec:con}}
Our paper proposes a novel framework that overcomes the limitations of existing federated GNN frameworks in local task featuring and task relating. We utilize graph distillation in task featuring, and introduce a novel kernelized attentive aggregation mechanism based on a collaborated network to incorporate global connectivity during model aggregation. The extensive experimental results demonstrate that our framework outperforms state-of-the-art methods.
\bibliographystyle{ACM-Reference-Format}

\begin{thebibliography}{49}


\ifx \showCODEN    \undefined \def \showCODEN     #1{\unskip}     \fi
\ifx \showDOI      \undefined \def \showDOI       #1{#1}\fi
\ifx \showISBNx    \undefined \def \showISBNx     #1{\unskip}     \fi
\ifx \showISBNxiii \undefined \def \showISBNxiii  #1{\unskip}     \fi
\ifx \showISSN     \undefined \def \showISSN      #1{\unskip}     \fi
\ifx \showLCCN     \undefined \def \showLCCN      #1{\unskip}     \fi
\ifx \shownote     \undefined \def \shownote      #1{#1}          \fi
\ifx \showarticletitle \undefined \def \showarticletitle #1{#1}   \fi
\ifx \showURL      \undefined \def \showURL       {\relax}        \fi
\providecommand\bibfield[2]{#2}
\providecommand\bibinfo[2]{#2}
\providecommand\natexlab[1]{#1}
\providecommand\showeprint[2][]{arXiv:#2}

\bibitem[Arivazhagan et~al\mbox{.}(2019)]%
        {arivazhagan2019federated}
\bibfield{author}{\bibinfo{person}{Manoj~Ghuhan Arivazhagan},
  \bibinfo{person}{Vinay Aggarwal}, \bibinfo{person}{Aaditya~Kumar Singh},
  {and} \bibinfo{person}{Sunav Choudhary}.} \bibinfo{year}{2019}\natexlab{}.
\newblock \showarticletitle{Federated learning with personalization layers}.
\newblock \bibinfo{journal}{\emph{arXiv preprint arXiv:1912.00818}}
  (\bibinfo{year}{2019}).
\newblock


\bibitem[Baek et~al\mbox{.}(2022)]%
        {baek2022personalized}
\bibfield{author}{\bibinfo{person}{Jinheon Baek}, \bibinfo{person}{Wonyong
  Jeong}, \bibinfo{person}{Jiongdao Jin}, \bibinfo{person}{Jaehong Yoon}, {and}
  \bibinfo{person}{Sung~Ju Hwang}.} \bibinfo{year}{2022}\natexlab{}.
\newblock \showarticletitle{Personalized Subgraph Federated Learning}. In
  \bibinfo{booktitle}{\emph{Proc. ICML}}.
\newblock


\bibitem[Bao et~al\mbox{.}(2023)]%
        {bao2023optimizing}
\bibfield{author}{\bibinfo{person}{Wenxuan Bao}, \bibinfo{person}{Haohan Wang},
  \bibinfo{person}{Jun Wu}, {and} \bibinfo{person}{Jingrui He}.}
  \bibinfo{year}{2023}\natexlab{}.
\newblock \showarticletitle{Optimizing the Collaboration Structure in
  Cross-Silo Federated Learning}. In \bibinfo{booktitle}{\emph{Proc. ICML}}.
\newblock


\bibitem[Black et~al\mbox{.}(2023)]%
        {black2023understanding}
\bibfield{author}{\bibinfo{person}{Mitchell Black}, \bibinfo{person}{Zhengchao
  Wan}, \bibinfo{person}{Amir Nayyeri}, {and} \bibinfo{person}{Yusu Wang}.}
  \bibinfo{year}{2023}\natexlab{}.
\newblock \showarticletitle{Understanding oversquashing in gnns through the
  lens of effective resistance}. In \bibinfo{booktitle}{\emph{Proc. ICML}}.
  PMLR, \bibinfo{pages}{2528--2547}.
\newblock


\bibitem[Chen et~al\mbox{.}(2022b)]%
        {chen2022structure}
\bibfield{author}{\bibinfo{person}{Dexiong Chen}, \bibinfo{person}{Leslie
  O’Bray}, {and} \bibinfo{person}{Karsten Borgwardt}.}
  \bibinfo{year}{2022}\natexlab{b}.
\newblock \showarticletitle{Structure-aware transformer for graph
  representation learning}. In \bibinfo{booktitle}{\emph{International
  Conference on Machine Learning}}. PMLR, \bibinfo{pages}{3469--3489}.
\newblock


\bibitem[Chen et~al\mbox{.}(2022a)]%
        {chen2022personalized}
\bibfield{author}{\bibinfo{person}{Fengwen Chen}, \bibinfo{person}{Guodong
  Long}, \bibinfo{person}{Zonghan Wu}, \bibinfo{person}{Tianyi Zhou}, {and}
  \bibinfo{person}{Jing Jiang}.} \bibinfo{year}{2022}\natexlab{a}.
\newblock \showarticletitle{Personalized federated learning with graph}. In
  \bibinfo{booktitle}{\emph{Proc. IJCAI}}.
\newblock


\bibitem[Chen et~al\mbox{.}(2021)]%
        {chen2021theorem}
\bibfield{author}{\bibinfo{person}{Shuxiao Chen}, \bibinfo{person}{Qinqing
  Zheng}, \bibinfo{person}{Qi Long}, {and} \bibinfo{person}{Weijie~J Su}.}
  \bibinfo{year}{2021}\natexlab{}.
\newblock \showarticletitle{A theorem of the alternative for personalized
  federated learning}.
\newblock \bibinfo{journal}{\emph{arXiv preprint arXiv:2103.01901}}
  (\bibinfo{year}{2021}).
\newblock


\bibitem[Choromanski et~al\mbox{.}(2020)]%
        {choromanski2020rethinking}
\bibfield{author}{\bibinfo{person}{Krzysztof Choromanski},
  \bibinfo{person}{Valerii Likhosherstov}, \bibinfo{person}{David Dohan},
  \bibinfo{person}{Xingyou Song}, \bibinfo{person}{Andreea Gane},
  \bibinfo{person}{Tamas Sarlos}, \bibinfo{person}{Peter Hawkins},
  \bibinfo{person}{Jared Davis}, \bibinfo{person}{Afroz Mohiuddin},
  \bibinfo{person}{Lukasz Kaiser}, {et~al\mbox{.}}}
  \bibinfo{year}{2020}\natexlab{}.
\newblock \showarticletitle{Rethinking attention with performers}.
\newblock \bibinfo{journal}{\emph{arXiv preprint arXiv:2009.14794}}
  (\bibinfo{year}{2020}).
\newblock


\bibitem[Dong et~al\mbox{.}(2022)]%
        {dong2022privacy}
\bibfield{author}{\bibinfo{person}{Tian Dong}, \bibinfo{person}{Bo Zhao}, {and}
  \bibinfo{person}{Lingjuan Lyu}.} \bibinfo{year}{2022}\natexlab{}.
\newblock \showarticletitle{Privacy for free: How does dataset condensation
  help privacy?}. In \bibinfo{booktitle}{\emph{International Conference on
  Machine Learning}}. PMLR, \bibinfo{pages}{5378--5396}.
\newblock


\bibitem[Duan and Wang(2022)]%
        {duan2022adaptive}
\bibfield{author}{\bibinfo{person}{Yaqi Duan} {and} \bibinfo{person}{Kaizheng
  Wang}.} \bibinfo{year}{2022}\natexlab{}.
\newblock \showarticletitle{Adaptive and robust multi-task learning}.
\newblock \bibinfo{journal}{\emph{arXiv preprint arXiv:2202.05250}}
  (\bibinfo{year}{2022}).
\newblock


\bibitem[Fey and Lenssen(2019)]%
        {fey2019fast}
\bibfield{author}{\bibinfo{person}{Matthias Fey} {and}
  \bibinfo{person}{Jan~Eric Lenssen}.} \bibinfo{year}{2019}\natexlab{}.
\newblock \showarticletitle{Fast graph representation learning with PyTorch
  Geometric}.
\newblock \bibinfo{journal}{\emph{arXiv preprint arXiv:1903.02428}}
  (\bibinfo{year}{2019}).
\newblock


\bibitem[Gilmer et~al\mbox{.}(2017)]%
        {pmlr-v70-gilmer17a}
\bibfield{author}{\bibinfo{person}{Justin Gilmer}, \bibinfo{person}{Samuel~S.
  Schoenholz}, \bibinfo{person}{Patrick~F. Riley}, \bibinfo{person}{Oriol
  Vinyals}, {and} \bibinfo{person}{George~E. Dahl}.}
  \bibinfo{year}{2017}\natexlab{}.
\newblock \showarticletitle{Neural Message Passing for Quantum Chemistry}. In
  \bibinfo{booktitle}{\emph{Proceedings of the 34th International Conference on
  Machine Learning}} \emph{(\bibinfo{series}{Proceedings of Machine Learning
  Research}, Vol.~\bibinfo{volume}{70})},
  \bibfield{editor}{\bibinfo{person}{Doina Precup} {and}
  \bibinfo{person}{Yee~Whye Teh}} (Eds.). \bibinfo{publisher}{PMLR},
  \bibinfo{address}{International Convention Centre, Sydney, Australia},
  \bibinfo{pages}{1263--1272}.
\newblock


\bibitem[Girvan and Newman(2002)]%
        {girvan2002community}
\bibfield{author}{\bibinfo{person}{Michelle Girvan} {and}
  \bibinfo{person}{Mark~EJ Newman}.} \bibinfo{year}{2002}\natexlab{}.
\newblock \showarticletitle{Community structure in social and biological
  networks}.
\newblock \bibinfo{journal}{\emph{Proceedings of the national academy of
  sciences}} \bibinfo{volume}{99}, \bibinfo{number}{12} (\bibinfo{year}{2002}),
  \bibinfo{pages}{7821--7826}.
\newblock


\bibitem[Hanneke and Kpotufe(2022)]%
        {hanneke2022no}
\bibfield{author}{\bibinfo{person}{Steve Hanneke} {and} \bibinfo{person}{Samory
  Kpotufe}.} \bibinfo{year}{2022}\natexlab{}.
\newblock \showarticletitle{A no-free-lunch theorem for multitask learning}.
\newblock \bibinfo{journal}{\emph{The Annals of Statistics}}
  \bibinfo{volume}{50}, \bibinfo{number}{6} (\bibinfo{year}{2022}),
  \bibinfo{pages}{3119--3143}.
\newblock


\bibitem[Hastie et~al\mbox{.}(2009)]%
        {hastie2009elements}
\bibfield{author}{\bibinfo{person}{Trevor Hastie}, \bibinfo{person}{Robert
  Tibshirani}, \bibinfo{person}{Jerome~H Friedman}, {and}
  \bibinfo{person}{Jerome~H Friedman}.} \bibinfo{year}{2009}\natexlab{}.
\newblock \bibinfo{booktitle}{\emph{The elements of statistical learning: data
  mining, inference, and prediction}}. Vol.~\bibinfo{volume}{2}.
\newblock \bibinfo{publisher}{Springer}.
\newblock


\bibitem[Horn and Johnson(2012)]%
        {horn2012matrix}
\bibfield{author}{\bibinfo{person}{Roger~A Horn} {and}
  \bibinfo{person}{Charles~R Johnson}.} \bibinfo{year}{2012}\natexlab{}.
\newblock \bibinfo{booktitle}{\emph{Matrix analysis}}.
\newblock \bibinfo{publisher}{Cambridge university press}.
\newblock


\bibitem[Hu et~al\mbox{.}(2020)]%
        {NEURIPS2020_fb60d411}
\bibfield{author}{\bibinfo{person}{Weihua Hu}, \bibinfo{person}{Matthias Fey},
  \bibinfo{person}{Marinka Zitnik}, \bibinfo{person}{Yuxiao Dong},
  \bibinfo{person}{Hongyu Ren}, \bibinfo{person}{Bowen Liu},
  \bibinfo{person}{Michele Catasta}, {and} \bibinfo{person}{Jure Leskovec}.}
  \bibinfo{year}{2020}\natexlab{}.
\newblock \showarticletitle{Open Graph Benchmark: Datasets for Machine Learning
  on Graphs}. In \bibinfo{booktitle}{\emph{Proc. NeurIPS}}.
  \bibinfo{pages}{22118--22133}.
\newblock


\bibitem[Jang et~al\mbox{.}(2016)]%
        {jang2016categorical}
\bibfield{author}{\bibinfo{person}{Eric Jang}, \bibinfo{person}{Shixiang Gu},
  {and} \bibinfo{person}{Ben Poole}.} \bibinfo{year}{2016}\natexlab{}.
\newblock \showarticletitle{Categorical reparameterization with
  gumbel-softmax}.
\newblock \bibinfo{journal}{\emph{arXiv preprint arXiv:1611.01144}}
  (\bibinfo{year}{2016}).
\newblock


\bibitem[Jin et~al\mbox{.}(2022a)]%
        {jin2022condensing}
\bibfield{author}{\bibinfo{person}{Wei Jin}, \bibinfo{person}{Xianfeng Tang},
  \bibinfo{person}{Haoming Jiang}, \bibinfo{person}{Zheng Li},
  \bibinfo{person}{Danqing Zhang}, \bibinfo{person}{Jiliang Tang}, {and}
  \bibinfo{person}{Bing Yin}.} \bibinfo{year}{2022}\natexlab{a}.
\newblock \showarticletitle{Condensing graphs via one-step gradient matching}.
  In \bibinfo{booktitle}{\emph{Proc. ACM SIGKDD}}. \bibinfo{pages}{720--730}.
\newblock


\bibitem[Jin et~al\mbox{.}(2022b)]%
        {jin2022graph}
\bibfield{author}{\bibinfo{person}{Wei Jin}, \bibinfo{person}{Lingxiao Zhao},
  \bibinfo{person}{Shichang Zhang}, \bibinfo{person}{Yozen Liu},
  \bibinfo{person}{Jiliang Tang}, {and} \bibinfo{person}{Neil Shah}.}
  \bibinfo{year}{2022}\natexlab{b}.
\newblock \showarticletitle{Graph Condensation for Graph Neural Networks}. In
  \bibinfo{booktitle}{\emph{Proc. ICLR}}.
\newblock


\bibitem[Kairouz et~al\mbox{.}(2021)]%
        {kairouz2021advances}
\bibfield{author}{\bibinfo{person}{Peter Kairouz}, \bibinfo{person}{H~Brendan
  McMahan}, \bibinfo{person}{Brendan Avent}, \bibinfo{person}{Aur{\'e}lien
  Bellet}, \bibinfo{person}{Mehdi Bennis}, \bibinfo{person}{Arjun~Nitin
  Bhagoji}, \bibinfo{person}{Kallista Bonawitz}, \bibinfo{person}{Zachary
  Charles}, \bibinfo{person}{Graham Cormode}, \bibinfo{person}{Rachel
  Cummings}, {et~al\mbox{.}}} \bibinfo{year}{2021}\natexlab{}.
\newblock \showarticletitle{Advances and open problems in federated learning}.
\newblock \bibinfo{journal}{\emph{Foundations and Trends{\textregistered} in
  Machine Learning}} \bibinfo{volume}{14}, \bibinfo{number}{1--2}
  (\bibinfo{year}{2021}), \bibinfo{pages}{1--210}.
\newblock


\bibitem[Karimireddy et~al\mbox{.}(2020)]%
        {karimireddy2020scaffold}
\bibfield{author}{\bibinfo{person}{Sai~Praneeth Karimireddy},
  \bibinfo{person}{Satyen Kale}, \bibinfo{person}{Mehryar Mohri},
  \bibinfo{person}{Sashank Reddi}, \bibinfo{person}{Sebastian Stich}, {and}
  \bibinfo{person}{Ananda~Theertha Suresh}.} \bibinfo{year}{2020}\natexlab{}.
\newblock \showarticletitle{Scaffold: Stochastic controlled averaging for
  federated learning}. In \bibinfo{booktitle}{\emph{International conference on
  machine learning}}. PMLR, \bibinfo{pages}{5132--5143}.
\newblock


\bibitem[Kingma and Ba(2014)]%
        {kingma2014adam}
\bibfield{author}{\bibinfo{person}{Diederik~P Kingma} {and}
  \bibinfo{person}{Jimmy Ba}.} \bibinfo{year}{2014}\natexlab{}.
\newblock \showarticletitle{Adam: A method for stochastic optimization}.
\newblock \bibinfo{journal}{\emph{arXiv preprint arXiv:1412.6980}}
  (\bibinfo{year}{2014}).
\newblock


\bibitem[Kipf and Welling(2017)]%
        {kipf2016semi}
\bibfield{author}{\bibinfo{person}{Thomas~N Kipf} {and} \bibinfo{person}{Max
  Welling}.} \bibinfo{year}{2017}\natexlab{}.
\newblock \showarticletitle{Semi-supervised classification with graph
  convolutional networks}. In \bibinfo{booktitle}{\emph{Proc. ICLR}}.
\newblock


\bibitem[Kondor and Lafferty(2002)]%
        {kondor2002diffusion}
\bibfield{author}{\bibinfo{person}{Risi~Imre Kondor} {and}
  \bibinfo{person}{John Lafferty}.} \bibinfo{year}{2002}\natexlab{}.
\newblock \showarticletitle{Diffusion kernels on graphs and other discrete
  structures}. In \bibinfo{booktitle}{\emph{Proc. ICML}},
  Vol.~\bibinfo{volume}{2002}. \bibinfo{pages}{315--322}.
\newblock


\bibitem[Lehmann and Casella(2006)]%
        {lehmann2006theory}
\bibfield{author}{\bibinfo{person}{Erich~L Lehmann} {and}
  \bibinfo{person}{George Casella}.} \bibinfo{year}{2006}\natexlab{}.
\newblock \bibinfo{booktitle}{\emph{Theory of point estimation}}.
\newblock \bibinfo{publisher}{Springer Science \& Business Media}.
\newblock


\bibitem[Li et~al\mbox{.}(2021)]%
        {li2021ditto}
\bibfield{author}{\bibinfo{person}{Tian Li}, \bibinfo{person}{Shengyuan Hu},
  \bibinfo{person}{Ahmad Beirami}, {and} \bibinfo{person}{Virginia Smith}.}
  \bibinfo{year}{2021}\natexlab{}.
\newblock \showarticletitle{Ditto: Fair and robust federated learning through
  personalization}. In \bibinfo{booktitle}{\emph{Proc. ICML}}. PMLR,
  \bibinfo{pages}{6357--6368}.
\newblock


\bibitem[Li et~al\mbox{.}(2020)]%
        {li2020federated}
\bibfield{author}{\bibinfo{person}{Tian Li}, \bibinfo{person}{Anit~Kumar Sahu},
  \bibinfo{person}{Manzil Zaheer}, \bibinfo{person}{Maziar Sanjabi},
  \bibinfo{person}{Ameet Talwalkar}, {and} \bibinfo{person}{Virginia Smith}.}
  \bibinfo{year}{2020}\natexlab{}.
\newblock \showarticletitle{Federated optimization in heterogeneous networks}.
  In \bibinfo{booktitle}{\emph{Proc. MLSys}}, Vol.~\bibinfo{volume}{2}.
  \bibinfo{pages}{429--450}.
\newblock


\bibitem[Long et~al\mbox{.}(2023)]%
        {long2023multi}
\bibfield{author}{\bibinfo{person}{Guodong Long}, \bibinfo{person}{Ming Xie},
  \bibinfo{person}{Tao Shen}, \bibinfo{person}{Tianyi Zhou},
  \bibinfo{person}{Xianzhi Wang}, {and} \bibinfo{person}{Jing Jiang}.}
  \bibinfo{year}{2023}\natexlab{}.
\newblock \showarticletitle{Multi-center federated learning: clients clustering
  for better personalization}.
\newblock \bibinfo{journal}{\emph{World Wide Web}} \bibinfo{volume}{26},
  \bibinfo{number}{1} (\bibinfo{year}{2023}), \bibinfo{pages}{481--500}.
\newblock


\bibitem[Maddison et~al\mbox{.}(2016)]%
        {maddison2016concrete}
\bibfield{author}{\bibinfo{person}{Chris~J Maddison}, \bibinfo{person}{Andriy
  Mnih}, {and} \bibinfo{person}{Yee~Whye Teh}.}
  \bibinfo{year}{2016}\natexlab{}.
\newblock \showarticletitle{The concrete distribution: A continuous relaxation
  of discrete random variables}.
\newblock \bibinfo{journal}{\emph{arXiv preprint arXiv:1611.00712}}
  (\bibinfo{year}{2016}).
\newblock


\bibitem[McMahan et~al\mbox{.}(2017)]%
        {mcmahan2017communication}
\bibfield{author}{\bibinfo{person}{Brendan McMahan}, \bibinfo{person}{Eider
  Moore}, \bibinfo{person}{Daniel Ramage}, \bibinfo{person}{Seth Hampson},
  {and} \bibinfo{person}{Blaise~Aguera y Arcas}.}
  \bibinfo{year}{2017}\natexlab{}.
\newblock \showarticletitle{Communication-efficient learning of deep networks
  from decentralized data}. In \bibinfo{booktitle}{\emph{Proc. AISTATS}}. PMLR,
  \bibinfo{pages}{1273--1282}.
\newblock


\bibitem[Pillutla et~al\mbox{.}(2022)]%
        {pillutla2022federated}
\bibfield{author}{\bibinfo{person}{Krishna Pillutla}, \bibinfo{person}{Kshitiz
  Malik}, \bibinfo{person}{Abdel-Rahman Mohamed}, \bibinfo{person}{Mike
  Rabbat}, \bibinfo{person}{Maziar Sanjabi}, {and} \bibinfo{person}{Lin Xiao}.}
  \bibinfo{year}{2022}\natexlab{}.
\newblock \showarticletitle{Federated learning with partial model
  personalization}. In \bibinfo{booktitle}{\emph{International Conference on
  Machine Learning}}. PMLR, \bibinfo{pages}{17716--17758}.
\newblock


\bibitem[Ramezani et~al\mbox{.}(2021)]%
        {ramezani2021learn}
\bibfield{author}{\bibinfo{person}{Morteza Ramezani}, \bibinfo{person}{Weilin
  Cong}, \bibinfo{person}{Mehrdad Mahdavi}, \bibinfo{person}{Mahmut Kandemir},
  {and} \bibinfo{person}{Anand Sivasubramaniam}.}
  \bibinfo{year}{2021}\natexlab{}.
\newblock \showarticletitle{Learn Locally, Correct Globally: A Distributed
  Algorithm for Training Graph Neural Networks}. In
  \bibinfo{booktitle}{\emph{International Conference on Learning
  Representations}}.
\newblock


\bibitem[Sattler et~al\mbox{.}(2020)]%
        {sattler2020clustered}
\bibfield{author}{\bibinfo{person}{Felix Sattler},
  \bibinfo{person}{Klaus-Robert M{\"u}ller}, {and} \bibinfo{person}{Wojciech
  Samek}.} \bibinfo{year}{2020}\natexlab{}.
\newblock \showarticletitle{Clustered federated learning: Model-agnostic
  distributed multitask optimization under privacy constraints}.
\newblock \bibinfo{journal}{\emph{IEEE transactions on neural networks and
  learning systems}} \bibinfo{volume}{32}, \bibinfo{number}{8}
  (\bibinfo{year}{2020}), \bibinfo{pages}{3710--3722}.
\newblock


\bibitem[Shchur et~al\mbox{.}(2018)]%
        {shchur2018pitfalls}
\bibfield{author}{\bibinfo{person}{Oleksandr Shchur},
  \bibinfo{person}{Maximilian Mumme}, \bibinfo{person}{Aleksandar Bojchevski},
  {and} \bibinfo{person}{Stephan G{\"u}nnemann}.}
  \bibinfo{year}{2018}\natexlab{}.
\newblock \showarticletitle{Pitfalls of graph neural network evaluation}.
\newblock \bibinfo{journal}{\emph{arXiv preprint arXiv:1811.05868}}
  (\bibinfo{year}{2018}).
\newblock


\bibitem[Smith et~al\mbox{.}(2017)]%
        {smith2017federated}
\bibfield{author}{\bibinfo{person}{Virginia Smith}, \bibinfo{person}{Chao-Kai
  Chiang}, \bibinfo{person}{Maziar Sanjabi}, {and} \bibinfo{person}{Ameet~S
  Talwalkar}.} \bibinfo{year}{2017}\natexlab{}.
\newblock \showarticletitle{Federated multi-task learning}.
\newblock \bibinfo{journal}{\emph{Proc. NeurIPS}}  \bibinfo{volume}{30}
  (\bibinfo{year}{2017}).
\newblock


\bibitem[Tan et~al\mbox{.}(2023)]%
        {tan2023federated}
\bibfield{author}{\bibinfo{person}{Yue Tan}, \bibinfo{person}{Yixin Liu},
  \bibinfo{person}{Guodong Long}, \bibinfo{person}{Jing Jiang},
  \bibinfo{person}{Qinghua Lu}, {and} \bibinfo{person}{Chengqi Zhang}.}
  \bibinfo{year}{2023}\natexlab{}.
\newblock \showarticletitle{Federated learning on non-iid graphs via structural
  knowledge sharing}. In \bibinfo{booktitle}{\emph{Proc. AAAI}},
  Vol.~\bibinfo{volume}{37}. \bibinfo{pages}{9953--9961}.
\newblock


\bibitem[Tsai et~al\mbox{.}(2019)]%
        {tsai2019transformer}
\bibfield{author}{\bibinfo{person}{Yao-Hung~Hubert Tsai},
  \bibinfo{person}{Shaojie Bai}, \bibinfo{person}{Makoto Yamada},
  \bibinfo{person}{Louis-Philippe Morency}, {and} \bibinfo{person}{Ruslan
  Salakhutdinov}.} \bibinfo{year}{2019}\natexlab{}.
\newblock \showarticletitle{Transformer Dissection: An Unified Understanding
  for Transformer’s Attention via the Lens of Kernel}. In
  \bibinfo{booktitle}{\emph{Proc. EMNLP}}.
\newblock


\bibitem[Vaswani et~al\mbox{.}(2017)]%
        {vaswani2017attention}
\bibfield{author}{\bibinfo{person}{Ashish Vaswani}, \bibinfo{person}{Noam
  Shazeer}, \bibinfo{person}{Niki Parmar}, \bibinfo{person}{Jakob Uszkoreit},
  \bibinfo{person}{Llion Jones}, \bibinfo{person}{Aidan~N Gomez},
  \bibinfo{person}{{\L}ukasz Kaiser}, {and} \bibinfo{person}{Illia
  Polosukhin}.} \bibinfo{year}{2017}\natexlab{}.
\newblock \showarticletitle{Attention is all you need}.
\newblock \bibinfo{journal}{\emph{Advances in neural information processing
  systems}}  \bibinfo{volume}{30} (\bibinfo{year}{2017}).
\newblock


\bibitem[Wang et~al\mbox{.}(2018)]%
        {wang2018dataset}
\bibfield{author}{\bibinfo{person}{Tongzhou Wang}, \bibinfo{person}{Jun-Yan
  Zhu}, \bibinfo{person}{Antonio Torralba}, {and} \bibinfo{person}{Alexei~A
  Efros}.} \bibinfo{year}{2018}\natexlab{}.
\newblock \showarticletitle{Dataset distillation}.
\newblock \bibinfo{journal}{\emph{arXiv preprint arXiv:1811.10959}}
  (\bibinfo{year}{2018}).
\newblock


\bibitem[Woodworth et~al\mbox{.}(2020)]%
        {woodworth2020local}
\bibfield{author}{\bibinfo{person}{Blake Woodworth},
  \bibinfo{person}{Kumar~Kshitij Patel}, \bibinfo{person}{Sebastian Stich},
  \bibinfo{person}{Zhen Dai}, \bibinfo{person}{Brian Bullins},
  \bibinfo{person}{Brendan Mcmahan}, \bibinfo{person}{Ohad Shamir}, {and}
  \bibinfo{person}{Nathan Srebro}.} \bibinfo{year}{2020}\natexlab{}.
\newblock \showarticletitle{Is local SGD better than minibatch SGD?}. In
  \bibinfo{booktitle}{\emph{International Conference on Machine Learning}}.
  PMLR, \bibinfo{pages}{10334--10343}.
\newblock


\bibitem[Xie et~al\mbox{.}(2021)]%
        {xie2021federated}
\bibfield{author}{\bibinfo{person}{Han Xie}, \bibinfo{person}{Jing Ma},
  \bibinfo{person}{Li Xiong}, {and} \bibinfo{person}{Carl Yang}.}
  \bibinfo{year}{2021}\natexlab{}.
\newblock \showarticletitle{Federated graph classification over non-iid
  graphs}. In \bibinfo{booktitle}{\emph{Proc. NeurIPS}},
  Vol.~\bibinfo{volume}{34}. \bibinfo{pages}{18839--18852}.
\newblock


\bibitem[Xu et~al\mbox{.}(2023)]%
        {xu2023personalized}
\bibfield{author}{\bibinfo{person}{Jian Xu}, \bibinfo{person}{Xinyi Tong},
  {and} \bibinfo{person}{Shao-Lun Huang}.} \bibinfo{year}{2023}\natexlab{}.
\newblock \showarticletitle{Personalized Federated Learning with Feature
  Alignment and Classifier Collaboration}. In \bibinfo{booktitle}{\emph{Proc.
  ICLR}}.
\newblock


\bibitem[Xu et~al\mbox{.}(2018)]%
        {xu2018powerful}
\bibfield{author}{\bibinfo{person}{Keyulu Xu}, \bibinfo{person}{Weihua Hu},
  \bibinfo{person}{Jure Leskovec}, {and} \bibinfo{person}{Stefanie Jegelka}.}
  \bibinfo{year}{2018}\natexlab{}.
\newblock \showarticletitle{How powerful are graph neural networks?}
\newblock \bibinfo{journal}{\emph{arXiv preprint arXiv:1810.00826}}
  (\bibinfo{year}{2018}).
\newblock


\bibitem[Yang et~al\mbox{.}(2016)]%
        {yang2016revisiting}
\bibfield{author}{\bibinfo{person}{Zhilin Yang}, \bibinfo{person}{William
  Cohen}, {and} \bibinfo{person}{Ruslan Salakhudinov}.}
  \bibinfo{year}{2016}\natexlab{}.
\newblock \showarticletitle{Revisiting semi-supervised learning with graph
  embeddings}. In \bibinfo{booktitle}{\emph{Proc. ICML}}. PMLR,
  \bibinfo{pages}{40--48}.
\newblock


\bibitem[Ye et~al\mbox{.}(2023)]%
        {ye2023personalized}
\bibfield{author}{\bibinfo{person}{Rui Ye}, \bibinfo{person}{Zhenyang Ni},
  \bibinfo{person}{Fangzhao Wu}, \bibinfo{person}{Siheng Chen}, {and}
  \bibinfo{person}{Yanfeng Wang}.} \bibinfo{year}{2023}\natexlab{}.
\newblock \showarticletitle{Personalized Federated Learning with Inferred
  Collaboration Graphs}. In \bibinfo{booktitle}{\emph{Proc. ICML}}.
\newblock


\bibitem[Zhang et~al\mbox{.}(2021)]%
        {zhang2021subgraph}
\bibfield{author}{\bibinfo{person}{Ke Zhang}, \bibinfo{person}{Carl Yang},
  \bibinfo{person}{Xiaoxiao Li}, \bibinfo{person}{Lichao Sun}, {and}
  \bibinfo{person}{Siu~Ming Yiu}.} \bibinfo{year}{2021}\natexlab{}.
\newblock \showarticletitle{Subgraph federated learning with missing neighbor
  generation}. In \bibinfo{booktitle}{\emph{Proc. NeurIPS}},
  Vol.~\bibinfo{volume}{34}. \bibinfo{pages}{6671--6682}.
\newblock


\bibitem[Zhao et~al\mbox{.}(2023)]%
        {pmlr-v195-zhao23b}
\bibfield{author}{\bibinfo{person}{Xuyang Zhao}, \bibinfo{person}{Huiyuan
  Wang}, {and} \bibinfo{person}{Wei Lin}.} \bibinfo{year}{2023}\natexlab{}.
\newblock \showarticletitle{The Aggregation–Heterogeneity Trade-off in
  Federated Learning}. In \bibinfo{booktitle}{\emph{Proceedings of Thirty Sixth
  Conference on Learning Theory}} \emph{(\bibinfo{series}{Proceedings of
  Machine Learning Research}, Vol.~\bibinfo{volume}{195})},
  \bibfield{editor}{\bibinfo{person}{Gergely Neu} {and}
  \bibinfo{person}{Lorenzo Rosasco}} (Eds.). \bibinfo{publisher}{PMLR},
  \bibinfo{pages}{5478--5502}.
\newblock
\urldef\tempurl%
\url{https://proceedings.mlr.press/v195/zhao23b.html}
\showURL{%
\tempurl}


\bibitem[Zhou et~al\mbox{.}(2020)]%
        {zhou2020distilled}
\bibfield{author}{\bibinfo{person}{Yanlin Zhou}, \bibinfo{person}{George Pu},
  \bibinfo{person}{Xiyao Ma}, \bibinfo{person}{Xiaolin Li}, {and}
  \bibinfo{person}{Dapeng Wu}.} \bibinfo{year}{2020}\natexlab{}.
\newblock \showarticletitle{Distilled one-shot federated learning}.
\newblock \bibinfo{journal}{\emph{arXiv preprint arXiv:2009.07999}}
  (\bibinfo{year}{2020}).
\newblock


\end{thebibliography}

\newpage 

\appendix
\onecolumn

\section*{Appendix}
\setcounter{secnumdepth}{2}

\section{Algorithm}\label{sec:alg}
We present the algorithm of FedGKD framework in Alg.\ref{alg:fedgkd}. 

\begin{multicols}{2}
\begin{algorithm}[H]
\caption{FedGKD Framework}\label{alg:fedgkd}
\begin{algorithmic}[1]
\Require Local datasets $\{ G_i\}$; 
\Ensure Personalized local models $\{W_i^T \}$;\\ 
\textbf{Server}
\For {$t=1,\cdots, T$}
\If {$t>1$}
\State Receive local GNN models $\{{\mathbf{W}_i^t} \}$ 
\State Run task feature extractor; 
\State Get distilled node features $\{ \mathbf{X}_i^{s,t}\} $ and their embeddings $\mathbf{H}_i^{s,t}$;
\State Construct a collaboration network $G^{c,t}$; 
\State Kernelize $G^{c,t}$ to get a matrix $\mathbf{S}$ using \eqref{Eq.kernel}; 
\State Aggregate local parameters into $\{\overline{\mathbf{W}_i^t}\}$
\EndIf
\State Send GNN models $\{\overline{\mathbf{W}_i^t} \}$ to each client; 
\State Request each client to train local models;
\EndFor
\quad \\ \\ 
\textbf{Client $i$}
\State Receive $\overline{\mathbf{W_i^t}}$ from the server; 
\State $W_i^t \leftarrow \overline{\mathbf{W_i^t}}$; 
\For{$e=1,\cdots, E_t$}
\State Update $\mathbf{W_i^t} $ with the loss function $\mathcal{L}(f(G_i; \mathbf{W}), \mathbf{y}_i) + \lambda ||\mathbf{W} - \overline{\mathbf{W_i^t}}||_2^2$;  
\EndFor
\end{algorithmic}
\end{algorithm}
\columnbreak

\begin{algorithm}[H]
\caption{Task Feature Extractor}\label{alg:extractor}
\begin{algorithmic}[1]
\Require Local model weights $\{ \mathbf{W}_i^{t}\}$; 
\Ensure Distilled node features $\{ \mathbf{X}_i^{s,t}\} $ and their embeddings $\mathbf{H}_i^{s,t}$;\\ 
\textbf{Server}
\State Receive local GNN models $\{{\mathbf{W}_i^t} \}$ from each client; 
\State Initialize node features $\mathbf{X}_0\sim \mathcal{N}(0,1)$; 
\State Initialize labels $\mathbf{y}_0$; 
\State Send $\mathbf{X}_0, \mathbf{y}_0$ to each client; 
\State Receive distilled node features $\{ \mathbf{X}_i^{s,t}\} $ and their embeddings $\mathbf{H}_i^{s,t}$ from each client; 
\\ \\ 
\textbf{Client $i$}
\State Receive $\mathbf{X}_0$ and $\mathbf{y}_0$ from the server; 
\State $\mathbf{X}^{s,t}_i\leftarrow \mathbf{X}_0$; $\mathbf{y}^{s,t}_i\leftarrow \mathbf{y}_0$;
\State Convert $\mathbf{y}^{s,t}_i$ into  one-encoding form; 
\For {$e=1,\cdots, E_d$}
\State Use Gumbel-softmax to sample edges using \eqref{Eq.cons_a}; 
\State Compute loss $\mathcal{L}(f(G_i^{s,t};\mathbf{W}_i^t),\mathbf{y}_i^{s,t})$;
\State Update $\mathbf{X}^{s,t}_i$ and $\mathbf{y}^{s,t}_i$;
\EndFor
\State Send $\mathbf{X}^{s,t}_i$ and their embeddings to the server; 
\end{algorithmic}
\end{algorithm}
\end{multicols}

\section{Experiment Setup\label{app:es}}

\subsection{Datasets\label{app:dataset}}
We display the graph statistics in Table \ref{tab:dataset}. The 6 graphs are of different scales and have different clustering coefficients. 

\begin{table*}[ht]
\centering
\begin{tabular}{lcccccc}
\toprule[2pt]
Dataset & Cora  & CiteSeer & PubMed & Amazon Photo & Amazon Computers & Obgn Arxiv \\
\midrule
\# Nodes & 2,485  & 2,120  & 19,717  & 7,487  & 13,381  & 169,343  \\
\# Edges & 10,138  & 7,358  & 88,648  & 238,086  & 491,556  & 2,315,598  \\
\# Features & 1,433 & 3,703 & 500 & 745   & 767   & 128 \\
\# Categories & 7     & 6     & 3     & 8     & 10    & 40  \\
Clustering & 0.238  & 0.170  & 0.062  & 0.411  & 0.351  & 0.226  \\
\bottomrule[2pt]
\end{tabular}%
\caption{Original Graph Statistics}
\label{tab:dataset}
\end{table*}
For non-overlapping and overlapping datasets, we display their average clustering coefficents and heterogeneity  measured by medians of the client-wise label Jensen–Shannon divergences. Generally, a larger number of clients in the federated system leads to more heterogeneous local graphs. 
\begin{table*}[ht]
  \centering
\begin{tabular}{lccccccccc}
\toprule[2pt]
Dataset & \multicolumn{3}{c}{Cora} & \multicolumn{3}{c}{CiteSeer} & \multicolumn{3}{c}{PubMed} \\
\midrule
\# Clients & 5     & 10    & 20    & 5     & 10    & 20    & 5     & 10    & 20  \\
\midrule
\# Nodes  & 497   & 249   & 124   & 424   & 212   & 106   & 3,943  & 1,972  & 986  \\
\# Edges & 1,871  & 891   & 424   & 1,411  & 673   & 327   & 16,411  & 7,557  & 3,616  \\
Clustering & 0.250  & 0.261  & 0.266  & 0.176  & 0.177  & 0.181  & 0.063  & 0.066  & 0.068  \\
Heterogeneity & 0.304  & 0.358  & 0.391  & 0.142  & 0.229  & 0.255  & 0.087  & 0.097  & 0.119  \\
\midrule[2pt]
Dataset & \multicolumn{3}{c}{Amazon Photo} & \multicolumn{3}{c}{Amazon Computers} & \multicolumn{3}{c}{Obgn Arxiv} \\
\midrule
\# Clients  & 5     & 10    & 20    & 5     & 10    & 20    & 5     & 10    & 20  \\
\midrule
\# Nodes  & 1,497  & 749   & 374   & 2,676  & 1,338  & 669   & 33,868  & 16,934  & 8,467  \\
\# Edges & 42,930  & 19,294  & 8,300  & 84,202  & 35,589  & 24,577  & 406,896  & 182,758  & 86,150  \\
Clustering & 0.431  & 0.458  & 0.478  &  0.383  & 0.405  &  0.418  &  0.245 & 0.255  & 0.268 \\
Heterogeneity & 0.470  & 0.546  & 0.568  & 0.350   &   0.373  & 0.460   & 0.372  &0.398  & 0.412 \\
\bottomrule[2pt]
\end{tabular}%
  \caption{Non-overlapping Dataset Statistics}
 \label{tab:dataset_no}%
\end{table*}%

To create federated datasets with overlapping vertices, we first use Metis graph partitioning algorithm to split the full graph into $\lfloor \frac{n}{5}\rfloor$ subgraphs. Then we randomly sample a half of the vertices from each subgraph for 5 times to create 5 graphs with overlapping vertices. After operating the same procedures for all the $\lfloor \frac{n}{5}\rfloor$ subgraphs, we consider the resulted $n$ graphs belonging to different community but with overlapping vertices as the $n$ local datasets. The dataset statistics of the overlapping datasets are shown in Table \ref{tab:dataset_o}. Overlapping datasets suffer from great losses of edges compared with the global graph. In general, a small number of clients leads to a large number of missing edges. 
\begin{table}[h]
    \centering
\begin{tabular}{lccccccccc}
\toprule[2pt]
Dataset & \multicolumn{3}{c}{Cora} & \multicolumn{3}{c}{CiteSeer} & \multicolumn{3}{c}{PubMed} \\
\midrule
\# Clients & 10    & 30    & 50    & 10    & 30    & 50    & 10    & 30    & 50  \\
\midrule
\# Nodes  & 621   & 207   & 124   & 530   & 176   & 106   & 4,929  & 1,643  & 986  \\
\# Edges & 752   & 246   & 141   & 590   & 184   & 104   & 6,494  & 2,090  & 1,158  \\
Clustering & 0.069  & 0.087  & 0.080  & 0.053  & 0.056  & 0.050  & 0.020  & 0.021  & 0.022  \\
Heterogeneity & 0.176  & 0.310  & 0.358  & 0.230  & 0.215  & 0.225  & 0.153  & 0.090  & 0.089  \\
\midrule[2pt]
Dataset & \multicolumn{3}{c}{Amazon Photo} & \multicolumn{3}{c}{Amazon Computers} & \multicolumn{3}{c}{Obgn Arxiv} \\
\midrule
\# Clients  & 10    & 30    & 50    & 10    & 30    & 50    & 10    & 30    & 50  \\
\midrule
\# Nodes  & 1,872  & 623   & 374   & 3,345  & 1,115  & 669   & 42,336  & 14,112  & 8,467 \\
\# Edges & 18,311  & 5,591  & 2,987  & 36,597  & 10,283  & 5,516  & 175,717  & 51,520  & 28,487 \\
Clustering & 0.291  & 0.297  & 0.307  & 0.256  & 0.258  & 0.268  & 0.122  & 0.127  & 0.131 \\
Heterogeneity & 0.310  & 0.416  & 0.535  & 0.316  & 0.422  & 0.344  & 0.371  &  0.363   &  0.507 \\
\bottomrule[2pt]
\end{tabular}%
    \caption{Overlapping Dataset Statistics}
    \label{tab:dataset_o}
\end{table}
\subsection{Baselines\label{app:base}}
We have compared the performance of our framework with 8 baselines. We will introduce the baselines briefly. 
\begin{itemize}
    \item Local: A local learning framework in which each client trains a model upon the local dataset without any collaboration. 
    \item FedAvg: Clients send the local weights to the server and the server averages all the weights to initiate next round training. 
    \item FedProx: Clients train the local models based on a local loss function added by a regularization term with importance hyper-parameter $0.001$. 
    \item FedPer: Clients only send the GCN layers to the server without uploading READOUT layer weights. 
    \item FedPub: Server computes $n$ graph embeddings from a randomized graph input to feature each task and aggregates the weights according to the graph embeddings. 
    \item FedSage: Clients learn a GraphSage model and aggregate weights according to FedAvg framework. 
    \item GCFL: Server cluster clients into groups so that clients belonging to the same groups aggregate weights together. We use the recommended parameters in the paper \cite{xie2021federated}.
    \item FedStar: Clients decouple the node embeddings into shared structural embeddings and local embeddings. Clients only upload the GNN weights generating structural embeddings for server to aggregate. We use the recommended parameters in the paper \cite{tan2023federated}. 
\end{itemize}

\subsection{Parameter Configuration\label{app:pc}}
We perform a hyperparameter tuning using a grid-search method within the following range: 
\begin{itemize}
    \item Learning rate: $\{0.005, 0.01 \}$;
    \item Number of local training epochs: $\{1,3,5,7\}$;
    \item Sparsity control $\gamma$: $\{ 0.001, 0.75, 1.5,2.5,5.0\} $;
    \item Temperature on element-wise exponential $\tau_s$: $\{1,3,5,7,9 \}$;
    \item Temperature on matrix exponential $\tau$: $\{0.05,0.1,0.25,0.5,0.75,1.0 \}$;
    \item Weights on proximal term $\lambda$: $\{10^{-5}, 10^{-3} \}$.
\end{itemize}

\section{Ablation Study on Dynamic Task Feature Extractor}\label{sec:static_vs_dynamci}
\begin{figure}[h]
\begin{minipage}{0.3\textwidth}
    \includegraphics[width=0.97\textwidth]{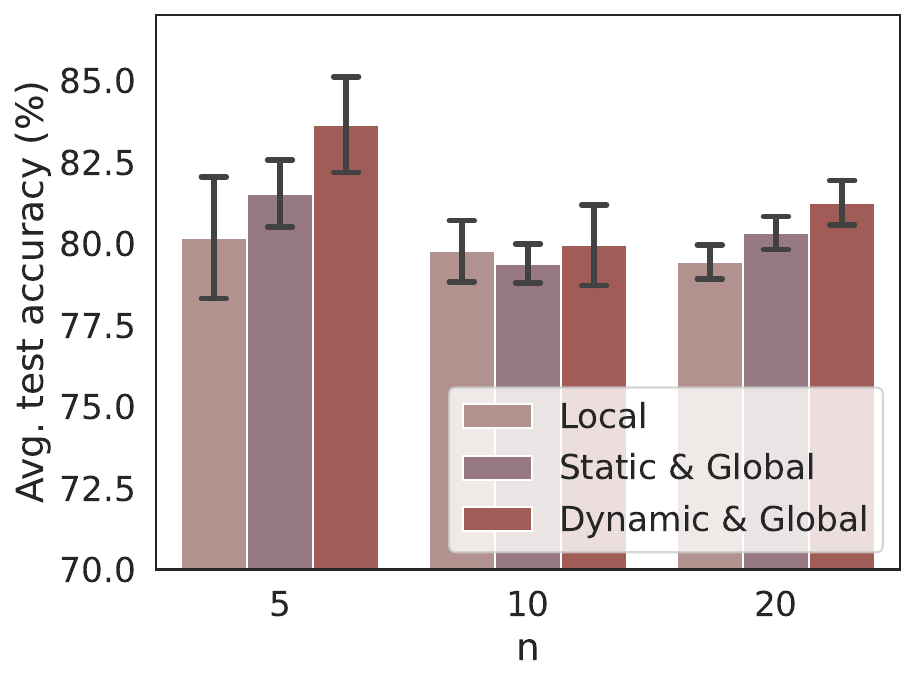}
    \caption{Effects of static and dynamic task feature extractor}
    \label{fig:dynamic}
\end{minipage}
\begin{minipage}{0.65\textwidth}
    We conducted a series of experiments in order to compare the effects of static and dynamic task feature extractors within our framework. These experiments were carried out on Cora dataset using non-overlapping settings. To accomplish this, we replaced our dynamic feature extractor, which distilled graphs based on the current model weights during each communication round, with a static extractor. This static extractor utilized a graph distillation method introduced in \cite{jin2022condensing} to ensure that the performance of the model training on the synthetic small graph closely approximated the "final" performance achieved when training on the original, larger graph before the federated learning procedures. In this way, the static extractor offers static representations of tasks, which remain unchanged throughout the entire federated learning process. As shown in Figure \ref{fig:dynamic}, the dynamic task feature extractor is visually replaced by the static extractor for our experimental purposes. Notably, this replacement results in a 2.4\% degradation in performance. The inferior performance of static task feature extractor comes from the unsatisfying local sample quality and quantity.
    Additionally, we compared the performance of the static extractor with a local training baseline and found that the static extractor still outperforms local learning.

\end{minipage}
\end{figure}

\section{On the validity of $k$ in (11)}\label{sec:proof}
Below we state the result that $k$ in \eqref{Eq.kernel} is a valid kernel:
\begin{proposition}
\label{prop:kernel}
    The function $k$ is a valid kernel over the domain $[n]$.
\end{proposition}
\begin{proof}
    Let $\{\lambda_1, \ldots, \lambda_n\}$ be the eigenvalues of $\mat{R}$. It then follows that the eigenvalues of $\mat{S}$ are $\{e^{t\lambda_1}, \ldots, e^{t\lambda_n}\}$ which are non-negative, therefore $\mat{S}$ is positive semidefinite. The proposition follows by \cite[Theorem 7.5.9]{horn2012matrix}
\end{proof}


\end{document}